\documentclass[sigconf]{acmart}
\AtBeginDocument{%
  }
\usepackage{amsmath}
\usepackage{amsthm}
\usepackage{multirow}
\usepackage{pifont}
\usepackage{tcolorbox}
\usepackage{float}

\copyrightyear{2026}
\acmYear{2026}
\setcopyright{cc}
\setcctype{by}
\acmDOI{10.1145/3767308.3836469}
\acmConference[MM '26]{Proceedings of the 34th ACM International Conference on Multimedia}{November 10--14, 2026}{Rio de Janeiro, Brazil}
\acmBooktitle{Proceedings of the 34th ACM International Conference on Multimedia (MM '26), November 10--14, 2026, Rio de Janeiro, Brazil}
\acmISBN{979-8-4007-2213-4/2026/11}
\makeatletter
\pretocmd{\@mkauthors@iii}{\exhyphenpenalty=10000\relax}{}{}
\makeatother

\begin{document}

\title[DocPO: Advancing Document Policy Optimization via Tailored Step-Aware Rewards]{DocPO: Advancing Document Policy Optimization\texorpdfstring{\\}{ }via Tailored Step-Aware Rewards}

\author{Yunhao Wang}
\correspondingauthor
\authornote{Yunhao Wang and Binghong Wu contributed equally to this work.}
\orcid{0009-0006-2313-2884}
\email{luciuswang@tencent.com}
\affiliation{%
  \institution{Tencent Hunyuan}
  \city{Beijing}
  \country{China}
}

\author{Binghong Wu}
\correspondingauthor
\authornotemark[1]
\orcid{0000-0002-3361-2260}
\email{bing-hong.wu@foxmail.com}
\affiliation{%
  \institution{Tencent Hunyuan}
  \city{Shanghai}
  \country{China}
}

\author{Zhenyu Huang}
\orcid{0009-0009-0345-5039}
\email{zhenyuhuang@tencent.com}
\affiliation{%
  \institution{Tencent Hunyuan}
  \city{Shenzhen}
  \country{China}
}

\author{Jiacheng Shi}
\orcid{0009-0002-5122-8614}
\email{shijiacheng21@mails.ucas.ac.cn}
\affiliation{%
  \institution{Tencent Hunyuan}
  \city{Shenzhen}
  \country{China}
}

\author{Shuo Huang}
\orcid{0009-0006-0166-3668}
\email{sanwushuosi@163.com}
\affiliation{%
  \institution{Tencent Hunyuan}
  \city{Shenzhen}
  \country{China}
}

\author{Tinghao Yu}
\orcid{0009-0004-7228-0192}
\email{maxwellyu@tencent.com}
\affiliation{%
  \institution{Tencent Hunyuan}
  \city{Beijing}
  \country{China}
}

\author{Feng Zhang}
\orcid{0009-0005-3591-7462}
\email{jayzhang@tencent.com}
\affiliation{%
  \institution{Tencent Hunyuan}
  \city{Beijing}
  \country{China}
}

\renewcommand{\shortauthors}{Wang et al.}

\begin{abstract}
Reinforcement learning (RL) for document parsing often relies on reference-based rewards rooted in edit distance (e.g., tree edit distance), yet it remains hard to optimize in the high-accuracy regime because such rewards become weakly discriminative: near-correct outputs receive very similar scores, providing limited learning signal for hard cases. We propose Step-Aware Annealing (SAA), a plug-and-play reward sharpening mechanism that progressively increases reward curvature during training, amplifying subtle quality differences among high-scoring samples while preserving stability in early learning. Built on SAA, we introduce DocPO, a document policy optimization framework with element-specific, reference-based rewards anchored by edit-distance signals: normalized string edit distance (NED) for text, tree edit distance similarity (TEDS) for tables, and a hybrid Rubric+edit reward for formulas. Experiments on OmniDocBench and DocElemHard show that SAA consistently improves GRPO-style RL across document elements over non-annealed rewards, without requiring additional human supervision for reward construction.
\end{abstract}



\maketitle

\section{Introduction}

Document parsing serves as a foundational task in both Vision Language Models (VLMs) and Document AI~\citep{cui2021document}. Its core objective is to precisely decode heterogeneous elements arranged in 2D document layouts, such as text, formula, and table, into 1D sequences. Crucially, this task transcends vanilla Optical Character Recognition (OCR) by necessitating a deep understanding of element-specific characteristics. Specifically, accurately resolving details such as \textit{varied formula expressions} and \textit{tables with nested cells} places stringent demands on models' fine-grained perception and sequence generation capabilities.

\begin{figure*}[!t]
    \centering
    \begin{minipage}[t]{0.425\textwidth}
        \vspace{0pt}
        \centering
        \includegraphics[width=\linewidth]{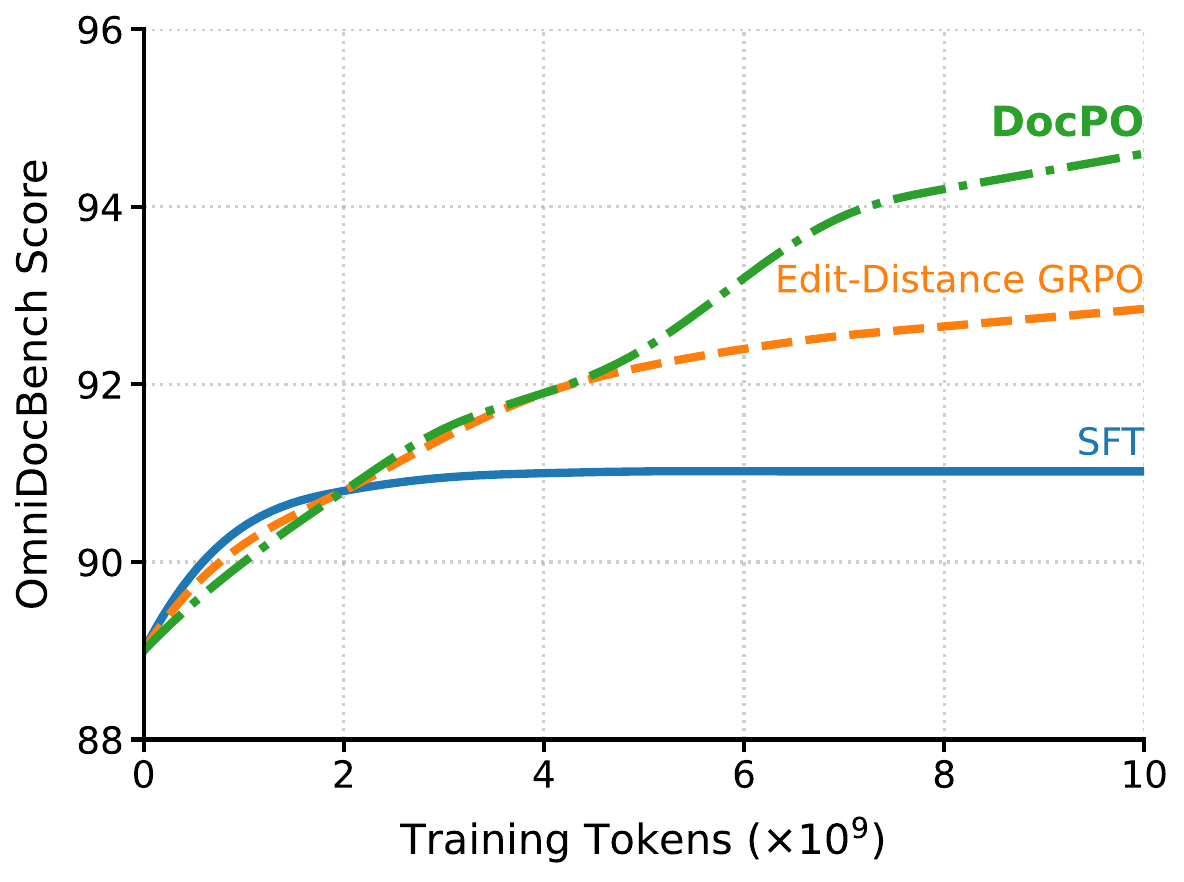}
        \\[-0.35em]
        {\small\textbf{(a) High-Accuracy RL Benefit}}
    \end{minipage}\hfill
    \begin{minipage}[t]{0.55\textwidth}
        \vspace{0pt}
        \centering
        \includegraphics[width=\linewidth]{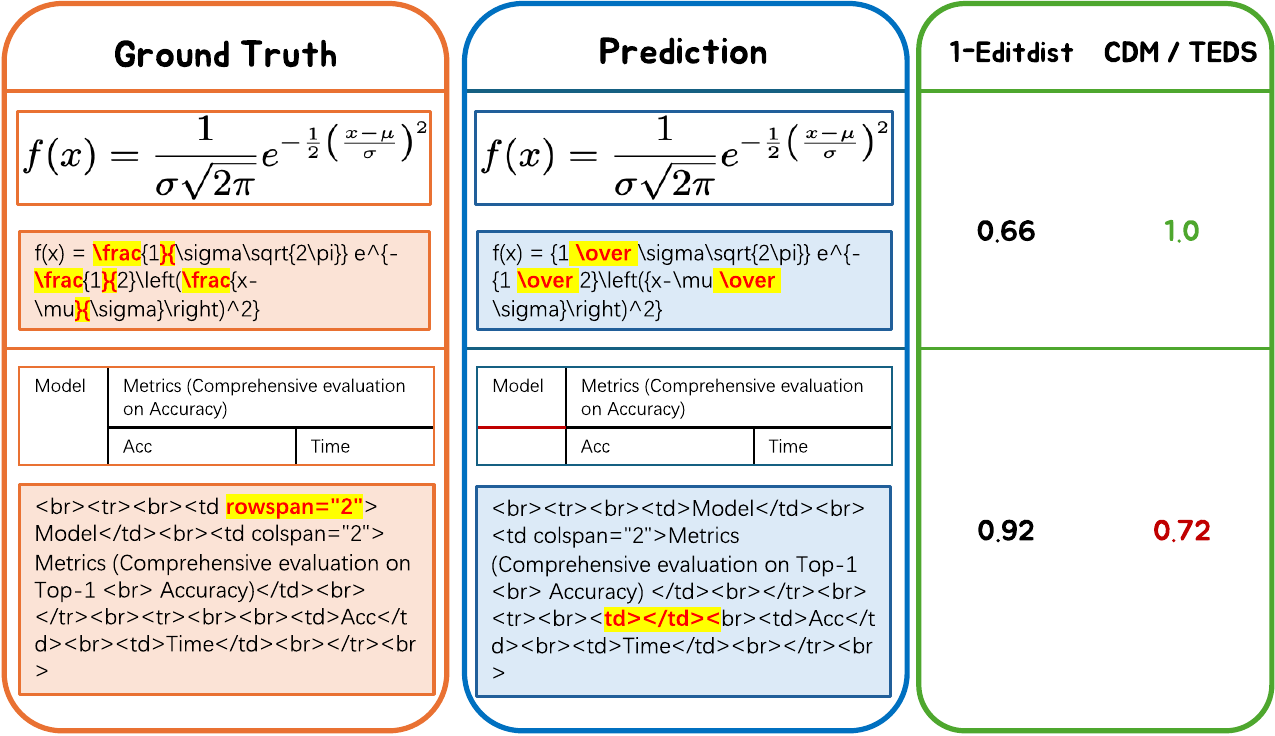}
        \\[-0.35em]
        {\small\textbf{(b) Failure Cases of String-Level Edit Distance}}
    \end{minipage}
    \caption{Two complementary views motivating DocPO. Left: Step-Aware Annealing improves optimization in the high-accuracy regime on OmniDocBench when applied to simple reference-based rewards. Right: string-level edit distance can mis-score outputs by penalizing semantically equivalent formulas and rewarding structurally broken tables with high lexical overlap.}
    \Description{A two-panel figure. The left panel is a training-curve comparison showing that Step-Aware Annealing improves optimization in the high-accuracy regime. The right panel shows two cases where string-level edit distance mis-scores document outputs: a semantically correct formula is penalized for lexical variation, and a structurally broken table is rewarded for lexical overlap.}
    \label{fig:teaser}
\end{figure*}

While VLMs have made remarkable progress in this field, current mainstream approaches still predominantly rely on the Supervised Fine-Tuning (SFT) paradigm~\citep{kim2022donut,blecher2023nougat}. Although SFT has achieved significant success driven by massive training data, the Teacher Forcing training mode can lead to Exposure Bias~\citep{zhang2025monkeyocr}. In response to these challenges, researchers have begun exploring Reinforcement Learning (RL), seeking improvements via sequence-level optimization. Yet, in stark contrast to its success in standard Large Language Models (LLMs)~\citep{hurst2024gpt,kumar2024training,guo2025deepseek}, RL's gains in OCR tasks remain relatively scarce. In our view, this discrepancy stems primarily from the misalignment between generic RL methods and the intrinsic characteristics of document elements.

Existing RL methods applied to document parsing~\citep{zhang2025monkeyocr,poznanski2025olmocr2,wang2025infinity} include two reward designs relevant to our setting. The first incorporates \textit{string-level} edit distance~\citep{wang2025infinity}, often treating heterogeneous elements uniformly with character-matching rewards. While this is a reasonable reference-based reward for plain text, it can misalign with rendered quality for formulas and with structural correctness for tables. As detailed in Figure~\ref{fig:teaser}(b), string-level edit distance generates false negatives for formulas by penalizing valid synonymous LaTeX variations (top), while yielding false positives for tables by rewarding high lexical overlap despite structural corruption (bottom). The second involves a learned \textit{render-and-compare} reward model~\citep{zhang2025monkeyocr}. Although this avoids manual HTML annotations, training an auxiliary reward model requires additional data and compute and is tailored to table optimization. Therefore, the pivotal question becomes: \textit{How can we retain simple element-appropriate reference-based rewards, make them sufficiently discriminative for RL, and avoid auxiliary reward-model training?}

We address this challenge with \textbf{Step-Aware Annealing} (SAA), a reward sharpening mechanism that can be applied to any normalized reference-based reward. Building on SAA, we introduce \textbf{DocPO}, a document policy optimization framework for document parsing. Our starting point is that practical document parsing already admits element-specific similarity measures in the \textit{edit-distance family}: normalized string edit distance (NED) for text, tree edit distance similarity (TEDS) for tables, and a \textit{Rubric+edit} reward for formulas. However, these base rewards become \textit{weakly discriminative} as the model approaches high accuracy: near-correct candidates receive very similar scores, which yields small advantages and slow convergence on hard cases. SAA addresses this issue through a training-step-dependent curvature schedule that progressively \textit{sharpens} the base reward while preserving the original ranking, making GRPO-style RL more effective without requiring additional human supervision for reward construction. In DocPO, the base reward remains element-specific, while SAA provides a unified optimization mechanism across text, tables, and formulas.

Notably, our focus is orthogonal to a dominant trend in recent specialized document VLMs: improving performance through modifications to the visual front-end, such as scaling dedicated high-resolution document encoders, redesigning encoders for aggressive vision-token compression, or introducing dynamic-resolution visual encoders with additional multimodal pretraining~\citep{li2025dotsocr,wei2025deepseek,cui2025paddleocrvl}. Such customization can be effective, but it may require extra adaptation and may reduce deployment simplicity. In contrast, DocPO keeps the standard Qwen2.5-VL-3B backbone unchanged, uses no extra pre-processing or post-processing modules, and improves performance solely through reward-level optimization.

We list our main contributions as follows:
\begin{itemize}
    \item \textbf{Step-Aware Annealing}: We propose a plug-and-play reward sharpening mechanism that anneals reward curvature over training steps, improving optimization in the high-accuracy regime under edit-distance-grounded reference-based rewards.
    
    \item \textbf{DocPO Framework}: We build DocPO as a document RL framework with task-tailored, reference-based rewards for text (NED), tables (TEDS), and formulas (syntax-gated \textit{Rubric\allowbreak{}+edit}).
    
    \item \textbf{Fine-Grained Element Benchmark}: We construct a more challenging fine-grained dataset covering text blocks, tables, and formulas for element-level document parsing.
\end{itemize}

\begin{figure*}[t]
    \centering
    \includegraphics[width=1\textwidth]{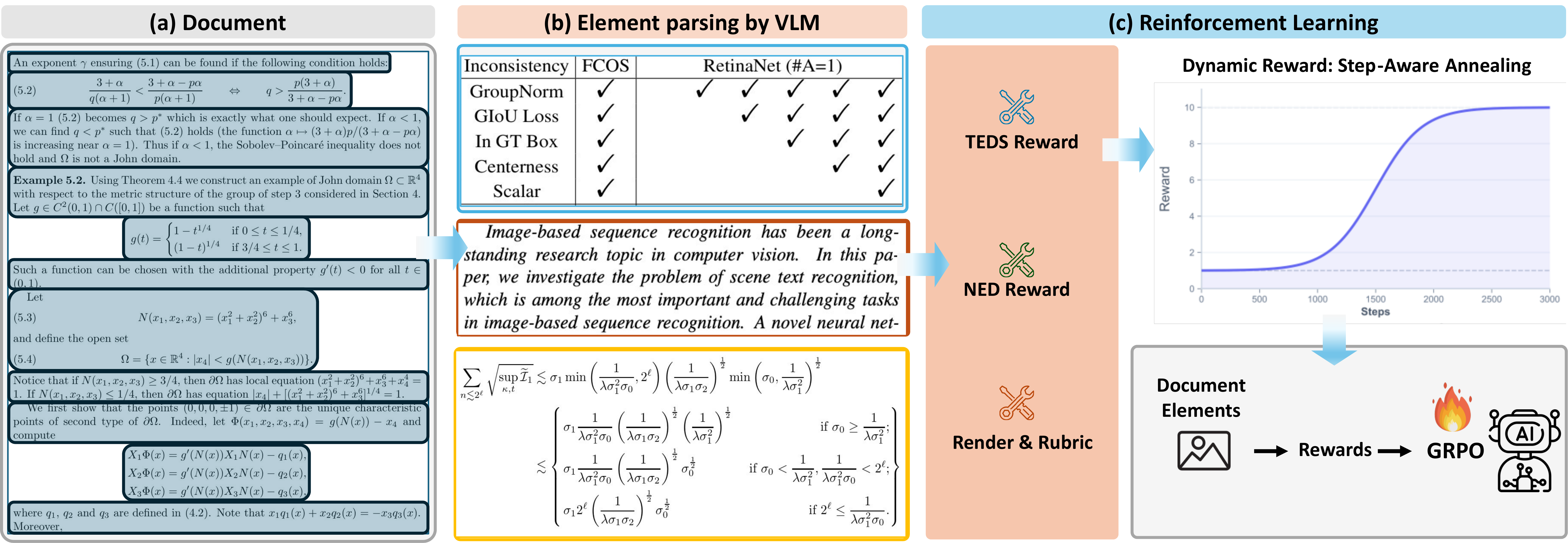} 
    \caption{DocPO pipeline. Element-specific rewards are normalized to $[0,1]$ and sharpened by Step-Aware Annealing during RL.}
    \Description{An overview of the DocPO pipeline. Document elements are parsed with element-specific base rewards, normalized to a shared scale, and then passed through the shared Step-Aware Annealing module during reinforcement learning.}
    \label{fig:workflow}
\end{figure*}

\begin{figure}[t]
    \centering
    \includegraphics[width=1\columnwidth]{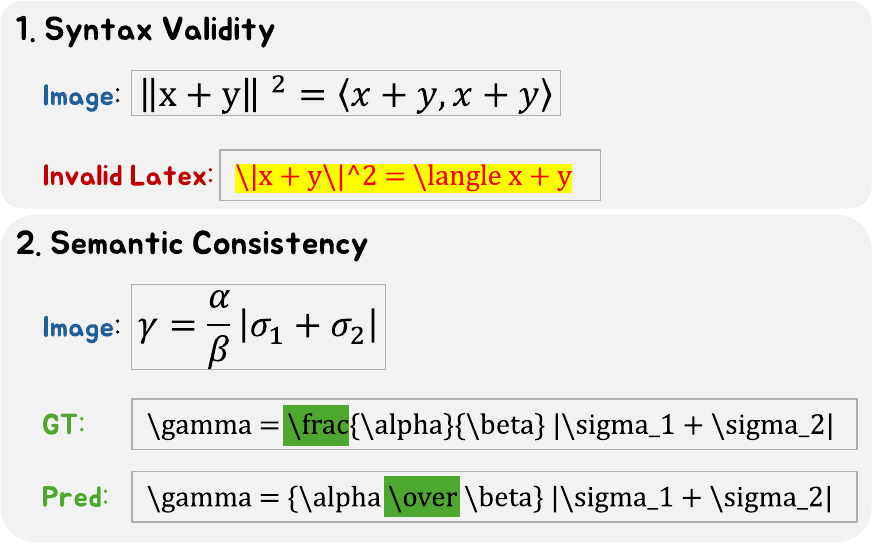}
    \caption{Examples of LaTeX evaluation challenges. The \textit{top case} demonstrates a \textbf{Syntax Validity} failure where the generated sequence is incomplete. The \textit{bottom case} highlights the need for \textbf{Semantic Consistency}, where the model predicts a valid alternative syntax (\texttt{\textbackslash over}) that differs lexically from the ground truth (\texttt{\textbackslash frac}) but remains mathematically correct.}
    \Description{Two formula examples illustrating evaluation challenges. One example shows an invalid incomplete formula, and the other shows two different LaTeX expressions that are semantically equivalent.}
    \label{fig:formula_case}
\end{figure}

\begin{table*}[htbp]
  \centering
    \begin{tabular}{lcccc}
      \toprule
      \textbf{Model} & \textbf{Method} & \textbf{Metric \& Strategy} & \textbf{Reward Type} & \textbf{Granularity} \\
      \midrule
      MonkeyOCR v1.5 & RL + RM & Render-and-compare & Continuous & Table \\
      INFINITY Parser & RLVR & Edit + layout/order & Continuous & Page \\
      olmOCR 2 & RLVR & Unit tests & Binary & Page \\
      \midrule
      \textbf{DocPO (ours)} & RLVR & Reference-based rewards + SAA & Continuous & Text / Formula / Table \\
      \bottomrule
   \end{tabular}
  \caption{Comparison of existing paradigms and our method. This motivates simple reference-based rewards together with Step-Aware Annealing, without requiring additional human supervision for reward construction.}
  \label{tab:model-comparison}
\end{table*}

\section{Related Work}

\subsection{Document Parsing Paradigms}
Existing methodologies in document parsing can be broadly categorized into three paradigms: pipelines, end-to-end VLMs, and modular VLMs.

\textbf{Pipelines}: Systems such as PaddleOCR~\citep{cui2025paddleocr}, MinerU~\citep{wang2024mineru}, and MonkeyOCR~\citep{li2025monkeyocr} typically involve a sequential workflow: detecting layout regions, extracting content via OCR, and linearizing results based on spatial layout~\citep{wang2021layoutreader,ha1995recursive}. The primary advantages of this paradigm are its efficiency and flexibility, often achieving a strong trade-off between performance and precision.

\textbf{End-to-End VLMs}: In contrast, end-to-end VLMs employ a single model to directly transcribe all document content into a linear sequence, as exemplified by Nougat~\citep{blecher2023nougat}, Kosmos-2.5~\citep{lv2023kosmos}, Qwen2.5-VL~\citep{bai2025qwen2}, olmOCR~\citep{poznanski2025olmocr}, DeepSeek-OCR~\citep{wei2025deepseek}, and the lightweight HunyuanOCR series~\citep{team2025hunyuanocr,li2026hunyuanocr}. Within a broader text-centric evaluation landscape exemplified by OCRBench v2~\citep{fu2025ocrbench,xu2026more}, related VLMs also advance text-centric visual understanding~\citep{tang2024textsquare,zhao2024harmonizing,lu2025bounding}, unify visual table tasks~\citep{zhao2024tabpedia}, or enable adaptive scene-text recognition~\citep{zhao2024multi}. Characterized by its simplicity, this approach requires only one model to complete the task and demonstrates robust adaptation across diverse scenarios, including camera-captured images. However, autoregressive decoding can cause latency on long or dense documents.

\textbf{Modular VLMs}: To address limitations of the aforementioned paradigms, recent systems such as Dolphin and MinerU2.5 introduce a modular design~\citep{feng2025dolphin,niu2025mineru2}. This approach utilizes distinct functional modules within a single model framework, typically adopting a \textit{crop-then-parse} strategy. By enabling element-level parallel decoding, it achieves a balance between usability and efficiency.

\textbf{Visual Front-End Customization}: Beyond output formulation, several recent document VLMs improve performance by altering the visual front-end. dots.ocr trains a dedicated 1.2B high-resolution encoder from scratch for native document inputs~\citep{li2025dotsocr}; DeepSeek-OCR introduces DeepEncoder, which serially combines window attention, a 16$\times$ convolutional compressor, and global attention to keep the vision-token budget manageable under high-resolution input~\citep{wei2025deepseek}; PaddleOCR-VL adopts a NaViT-style dynamic-resolution encoder initialized from Keye-VL~\citep{yang2025keyevl} and further adapted through large-scale multimodal pretraining~\citep{cui2025paddleocrvl}. By contrast, our work leaves the general-purpose Qwen2.5-VL backbone unchanged and studies optimization at the reward level.

\subsection{Reinforcement Learning for Document Parsing}
Despite the architectural diversity of these paradigms, their training predominantly relies on SFT. Consequently, RL for document parsing remains in its nascent stages. Beyond document parsing, progressive hard-case mining provides a related precedent by adaptively emphasizing difficult samples during object-detector training~\citep{wu2021progressive}. We compare representative RL methodologies in Table~\ref{tab:model-comparison}, with specific limitations detailed below.

\textbf{Learned Reward Models}: Targeting complex table recognition, MonkeyOCR v1.5~\citep{zhang2025monkeyocr} eliminates reliance on human-annotated HTML ground truth. It adopts a \textit{render-and-compare} strategy, where generated HTML is rendered into an image and evaluated against the original document by a specifically trained reward model (RM). While this mechanism prioritizes structural visual consistency over absolute character matching, it requires a separate, data-intensive RM training stage and is confined to table optimization, overlooking text and formulas.

\textbf{Reference-Based Rewards}: Recent RLVR approaches such as INFINITY Parser~\citep{wang2025infinity} optimize a composite reward of normalized edit distance, paragraph-count accuracy, and reading-order preservation without training an additional reward model. This design provides simple and scalable reference-based signals, but its string-similarity component can still misalign with rendered formula quality and table structure; StrucTab instead decomposes table rewards into validity, structure, and content~\citep{li2026structab}. For instance, a missing LaTeX symbol may incur only a limited character-level penalty while producing a visibly incorrect rendering.

\textbf{Binary Rewards}: Unlike continuous metrics, olmOCR 2~\citep{poznanski2025olmocr2} uses binary unit tests as rewards. With large-scale synthetic documents, it reports strong OCR-benchmark results while providing a sparse reward signal compared with continuous element-level metrics.

\begin{figure}[t]
    \centering
    \includegraphics[width=1\columnwidth]{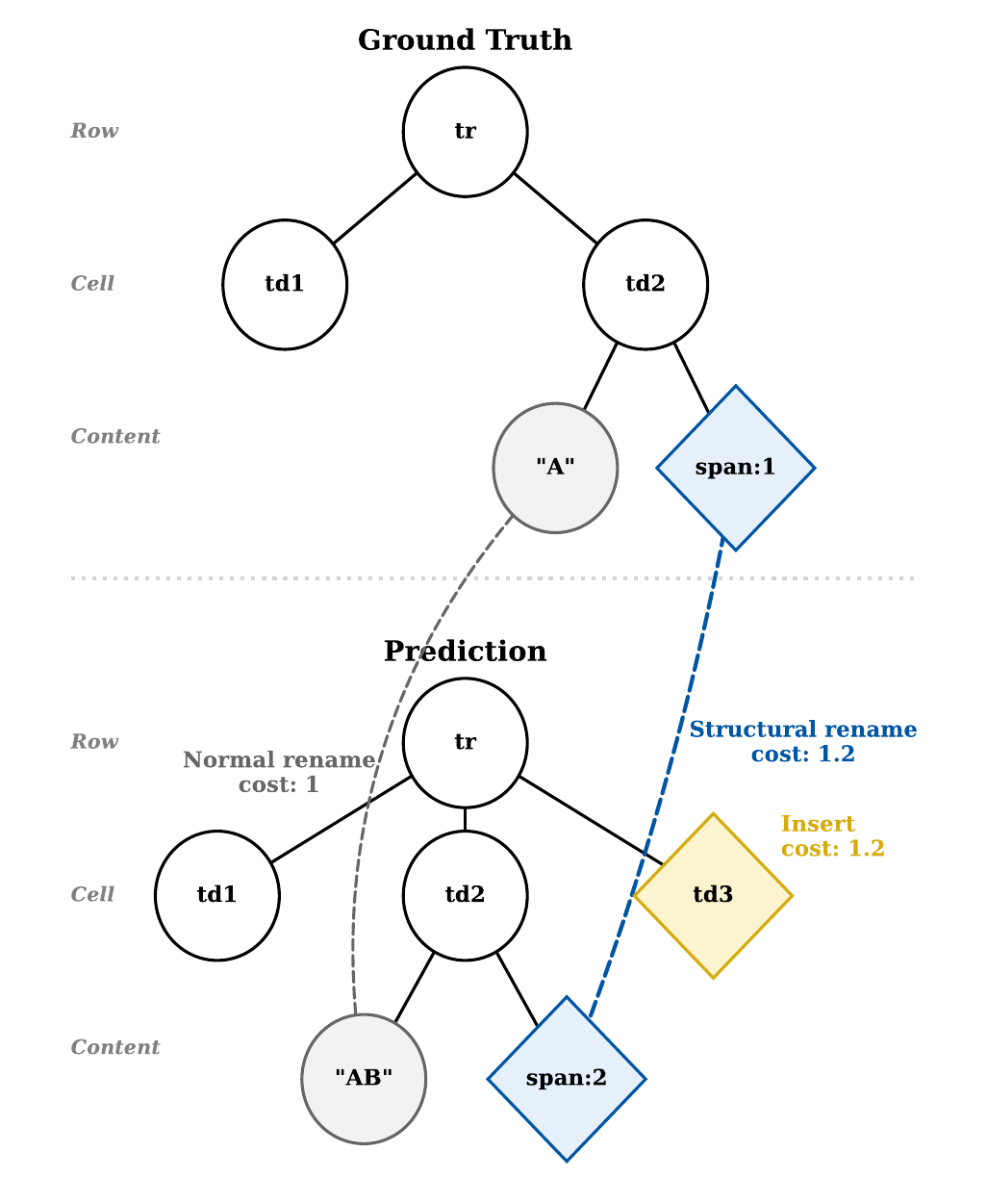}
    \caption{Illustration of weighted tree edit distance-based similarity. Distinct costs apply to structural renaming and insertion to evaluate structural and content consistency.}
    \Description{A diagram illustrating weighted tree edit distance for tables, with different edit operations receiving different costs to capture both structure and content fidelity.}
    \label{fig:tree_edit}
\end{figure}

\begin{figure}[t]
    \centering
    \includegraphics[width=1\columnwidth]{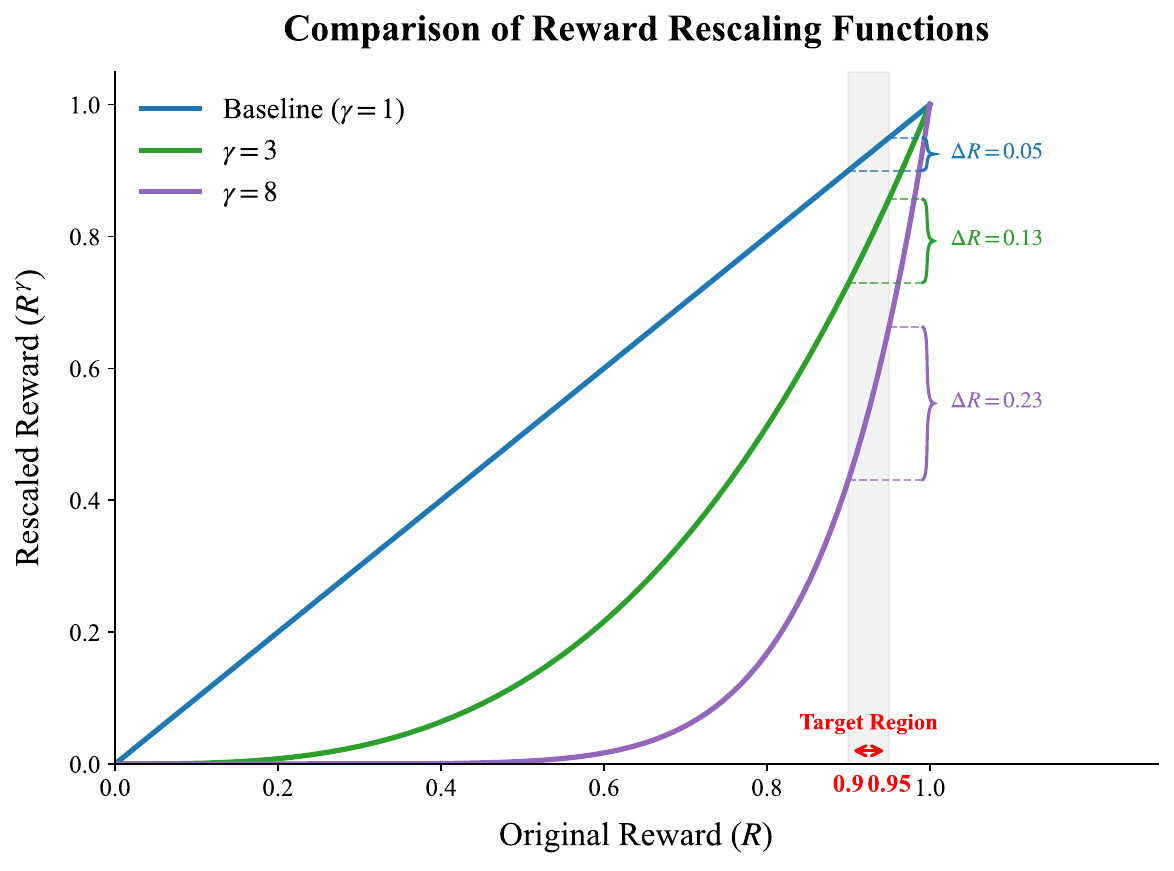}
    \caption{Comparison of reward rescaling functions with curvature factors ($\gamma$). Non-linear scaling (e.g., $\gamma=8$) amplifies the reward difference ($\Delta R$) in the high-score region, providing sharper optimization signals for fine-grained improvements.}
    \Description{A plot comparing reward rescaling curves under different curvature factors, showing stronger separation among high-score samples as the curvature increases.}
    \label{fig:annealing}
\end{figure}

\begin{figure}[t]
    \centering
    \includegraphics[width=1\columnwidth]{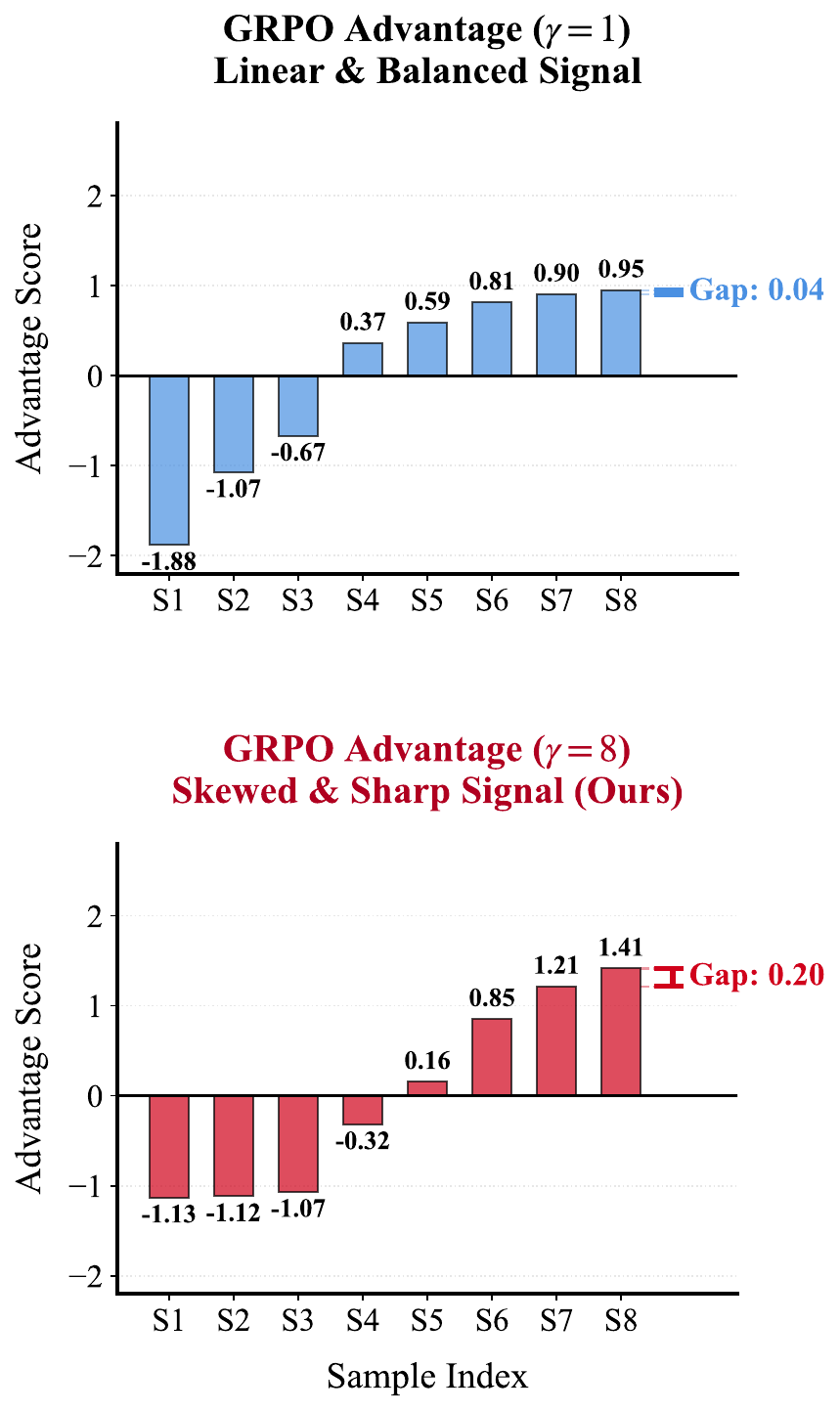}
    \caption{Comparison of advantage distributions for $\gamma=1$ (linear) and $\gamma=8$ (skewed). The higher $\gamma$ value sharpens the signal by increasing separation among the top-ranked samples, especially the top-1 and top-2 cases, rather than treating all non-top samples uniformly.}
    \Description{A comparison of advantage distributions under linear and sharpened reward scaling, showing that higher curvature increases discrimination among the highest-ranked samples, especially the top-1 and top-2 cases.}
    \label{fig:box_plot}
\end{figure}

\section{Base Reward Instantiation}
As illustrated in Figure~\ref{fig:workflow}, we use simple element-specific base rewards and map them to a shared $[0,1]$ scale before applying \textbf{Step-Aware Annealing} (Section~\ref{sec:step_aware_annealing}). These metrics reflect each element's native structure while retaining reference-based supervision, allowing the same annealing mechanism to operate without changing what each task considers correct. Throughout this section, $y$ and $y^*$ denote the prediction and reference.

\subsection{Base Reward for Table Recognition}
\label{sec:table_reward}

Tables are inherently hierarchical: an HTML table is a tree whose internal nodes encode structural semantics (row groupings, column spans) while leaf nodes carry cell content. A string-level comparison would ignore this hierarchy---two tables sharing most cell text but differing in a single \texttt{rowspan} can render as completely different layouts. We therefore adopt a \textit{weighted} tree edit distance similarity (TEDS) via the APTED algorithm~\citep{pawlik2016tree, pawlik2015efficient}, with differentiated costs (Figure~\ref{fig:tree_edit}): structural operations (node insertion/deletion, span renaming) receive a cost of 2, whereas content-only mismatches receive a cost of 1.

\begin{equation}
R_{\text{table}} = \mathrm{TEDS}_{\text{weighted}}(y, y^*) = 1 - \frac{\mathrm{EditDist}_{\text{weighted}}(y, y^*)}{\max\bigl(|y|, |y^*|\bigr)}
\end{equation}

\noindent This ensures that a single structural error (e.g., a missing \texttt{colspan}) incurs a heavier penalty than several character-level typos, aligning the reward with the perceptual importance of table layout.

\subsection{Base Reward for Text Recognition}
\label{sec:text_reward}

For text blocks, we adopt the complement of normalized edit distance (NED):
\begin{equation}
R_{\text{text}} = 1 - \mathrm{NED}(y, y^*) = 1 - \frac{\mathrm{EditDist}(y, y^*)}{\max\bigl(|y|, |y^*|\bigr)}
\end{equation}
This provides a lightweight, dense signal: scores near 1.0 indicate nearly perfect transcription. Character-level NED handles diverse scripts and punctuation without tokenization assumptions.

\subsection{Base Reward for Formula Recognition}
\label{sec:formula_reward}

Formulas present two co-existing challenges (Figure~\ref{fig:formula_case}): (1) an incomplete or malformed LaTeX sequence should receive zero credit regardless of partial overlap (\textit{syntactic validity}); (2) valid alternative notations (e.g., \texttt{\textbackslash frac\{a\}\{b\}} vs.\ \texttt{a \textbackslash over b}) should not be penalized (\textit{semantic equivalence}). We address both with a syntax-gated hybrid reward:
\begin{equation}
R_{\text{formula}} = v_{\text{syn}} \cdot \left( \alpha \cdot r_{\text{sem}} + \beta \cdot r_{\text{struct}} \right),
\quad r_{\text{struct}} = 1 - \mathrm{NED}(y, y^*)
\end{equation}
where $v_{\text{syn}} \in \{0, 1\}$ is a deterministic syntax gate (hard mask on compilation failure), $r_{\text{sem}} \in \{0,1\}$ is a binary rubric output for semantic equivalence, and $r_{\text{struct}}$ provides dense, continuous supervision via edit distance. We instantiate the rubric with Qwen2.5-7B-Instruct in a zero-shot setting. The weights satisfy $\alpha > \beta$ and $\alpha + \beta = 1$; we use $\alpha=0.8$ and $\beta=0.2$ in all experiments, prioritizing semantic correctness over literal form. The rubric captures meaning while NED remains dense and notation-sensitive, providing a practical signal. Prompt and reliability details are in the supplement.

\subsection{Step-Aware Annealing}
\label{sec:step_aware_annealing}

To address the challenge of weak reward discriminability in the high-accuracy regime, we propose \textbf{Step-Aware Annealing (SAA)}. SAA dynamically modulates reward curvature over \textit{training steps}, so that small differences among high-quality candidates are progressively amplified while early training remains stable. Crucially, SAA is agnostic to the particular reward definition once the base reward is normalized to $[0,1]$, which lets the same mechanism operate across text, tables, and formulas. In this section, we detail the non-linear reward shaping and the adaptive scheduling of $\gamma$.

\textbf{Non-Linear Reward Shaping}: At the core of our method lies a power-law transformation of the base metric (see Figure~\ref{fig:annealing}):

\begin{equation}
    f_{\gamma}(M) = M^{\gamma}
\end{equation}

\noindent where $M \in \{R_{\text{formula}}, R_{\text{table}}, R_{\text{text}}\}$ represents the base reward (all normalized to $[0,1]$). $\gamma$ is the curvature factor. When $\gamma>1$, the relative gap between two nearby scores is amplified: for $m\in(0,1)$ and small $\delta>0$, the ratio $(f_{\gamma}(m+\delta)-f_{\gamma}(m))/f_{\gamma}(m) = (1+\delta/m)^{\gamma}-1$ grows monotonically with $\gamma$, meaning near-ties become progressively easier to distinguish.

This separates \textit{what} is rewarded from \textit{how strongly} near-ties are distinguished: the base reward defines correctness, while $\gamma$ controls how aggressively those differences drive optimization. Formally, the relative gap $\rho_{\gamma}=(1+\delta/m)^{\gamma}-1$ is strictly increasing in $\gamma$ (see the supplement for a formal statement and proof), confirming that the sharpening effect strengthens monotonically. We anneal $\gamma$ from near-linear to more discriminative values over training.

\textbf{Interpretation}: For any base reward $M\in(0,1]$, we can rewrite $M^{\gamma}=\exp(\beta\log M)$ with inverse temperature $\beta=\gamma$ (equivalently, temperature $T=1/\gamma$). Increasing $\gamma$ therefore lowers the temperature, making the exponential weighting over $\log M$ more peaked and increasing discrimination among the highest-reward samples within each rollout group, especially the top-ranked candidates. This is analogous to an annealing schedule that transitions from soft (high-temperature) to sharp (low-temperature) selection (see the supplementary material for formal propositions on relative margin amplification and reward concentration).

\textbf{Adaptive Scheduling of $\gamma$}: To balance training stability with the sharpened discrimination power illustrated in Figure~\ref{fig:box_plot}, $\gamma$ is not fixed but adaptively updated based on runtime statistics. The update rule is defined as:

\begin{equation}
    \gamma = \gamma_{\text{init}} + \Delta_{\gamma} \cdot \left[ 1 - \exp\left( -\dfrac{s}{\tau_{\text{adaptive}}} \right) \right]
\end{equation}

Here, $\gamma_{\text{init}}$ sets the baseline amplification, $\Delta_{\gamma}$ controls the maximum adjustment range, and $s$ denotes the current training step.

In all experiments, we use a shared default setting of $\gamma_{\text{init}}=1$ and $\Delta_{\gamma}=8$. Empirically, $\gamma_{\text{init}}=1$ starts from the identity transform and keeps early updates aligned with the original base reward, while $\Delta_{\gamma}=8$ provides strong late-stage sharpening without destabilizing training. The supplement provides a sensitivity analysis over $\Delta_{\gamma}$.

To adapt the sharpening rate to runtime difficulty, we further introduce a task-wise Dynamic Dispersion Controller (DDC), which sets the adaptive time scale $\tau_{\text{adaptive}}$ using a task-wise dispersion score $d_{\text{DDC}}$ over recent rewards:
\begin{equation}
    \tau_{\text{adaptive}} = 
    \begin{cases} 
    \tau, & \text{if } s < s_{window} \\[5pt]
    \dfrac{\tau}{1 + d_{\text{DDC}}}, & \text{if } s \ge s_{window}
    \end{cases}
\end{equation}
where $s_{window}$ is the backtracking window size (set to 3). The controller defines $d_{\text{DDC}} = \frac{\sigma}{\mu + \epsilon}$ as a normalized dispersion score over the recent reward window. This score is instantiated with the rolling coefficient of variation, but it is used here as an internal control signal of DDC rather than as a standalone off-the-shelf module. In mixed-task RL, $d_{\text{DDC}}$ is computed separately within each task type (text, table, or formula), so DDC adapts the annealing rate using task-specific reward dispersion rather than a cross-task signal.

This piecewise design avoids noisy estimates early in training: in the early phase ($s < s_{window}$), a fixed $\tau$ prevents abnormal scaling due to limited data; in later phases, DDC responds to recent reward dispersion through $d_{\text{DDC}}$. Larger $d_{\text{DDC}}$ implies more heterogeneous reward outcomes and triggers faster sharpening, whereas smaller values keep the schedule closer to the default pace.

\section{Comparative Analysis}

We report experimental settings, benchmark results, and ablations isolating the contribution of SAA; additional training and filtering details appear in the supplement.

\subsection{Experimental Settings}\label{subsec:settings}

\begin{table}[t]
  \centering
  \caption{Task-level ablation on OmniDocBench and DocElemHard.}
  \label{tab:performance}
  \small
  \setlength{\tabcolsep}{4pt}
  \begin{tabular}{llcc}
    \toprule
    \textbf{Task} & \textbf{Method} & \textbf{OmniDocBench} & \textbf{DocElemHard} \\
    \midrule
    \multirow{3}{*}{Text $\downarrow$}
      & Baseline (w/o RL) & 0.0358 & 0.0910 \\
      & Edit-Dist Reward & 0.0238 & 0.0330 \\
      & Edit-Dist + SAA & \textbf{0.0125} & \textbf{0.0220} \\
    \midrule
    \multirow{5}{*}{Formula $\uparrow$}
      & Baseline (w/o RL) & 92.61 & 86.69 \\
      & Edit-Dist Reward & 92.86 & 85.89 \\
      & Rubric+edit Reward & 93.93 & 87.69 \\
      & \shortstack[l]{Rubric+edit\\(w/o syntax gate)} & 92.93 & 86.61 \\
      & Rubric+edit + SAA & \textbf{94.70} & \textbf{92.88} \\
    \midrule
    \multirow{4}{*}{Table $\uparrow$}
      & Baseline (w/o RL) & 89.30 & 83.20 \\
      & Edit-Dist Reward & 90.05 & 86.01 \\
      & APTED Reward & 91.70 & 87.21 \\
      & APTED + SAA & \textbf{93.01} & \textbf{90.60} \\
    \bottomrule
  \end{tabular}
\end{table}

\noindent\textit{Note:} For text, lower is better; for formula and table, higher is better. The extra Formula row isolates the effect of the syntax gate. Static-exponent variants for table recognition appear in Table~\ref{tab:table_static_gamma}. Table~\ref{tab:performance} is a task-level ablation; Table~\ref{tab:comparison} reports the final unified model.

\textbf{Datasets}: We constructed two distinct datasets to enhance parsing performance across different granularities. For the initial SFT stage, we utilize 490k full-page document samples with coarse-grained Mathpix annotations. For the subsequent fine-grained RL stage, we curated a high-precision dataset comprising 612k element patches. This RL set spans three structural categories: (1) \textit{RL-Tables} (206k samples), combining 86k manually annotated high-quality entries with filtered synthetic data; (2) \textit{RL-Formulas} (196k samples), sourced from open datasets and LaTeX rendering; and (3) \textit{RL-Text Blocks} (210k samples), incorporating open-source data and hard-case examples.

\textbf{Baseline Model}: Unless otherwise specified, the base model adopts Qwen2.5-VL-3B. For the scaling study in Table~\ref{tab:scaling_7b}, we also instantiate the same DocPO recipe on Qwen2.5-VL-7B. The shared SFT setup uses a maximum sequence length of 12k, global batch size of 512, constant learning rate of 3e-5, and 1 training epoch.

\textbf{RL Settings}: RL is performed on the baseline model with the following training parameters: input sequence length of 4k, output sequence length of 8k, global batch size of 128, constant learning rate of 1e-6, rollout number of 8, and no KL divergence constraint (prioritizing structural alignment performance over policy conservatism). Training concludes when the training reward plateaus.

\textbf{Training Strategy}: We train a unified model by mixing text, table, and formula patches. Each sample uses its base reward (NED/\allowbreak{}TEDS/\allowbreak{}\textit{Rubric+edit}) and is sharpened with Step-Aware Annealing.

\textbf{Evaluation}: Evaluation is performed on the OmniDocBench dataset~\citep{ouyang2025omnidocbench}, which contains 1,355 pages, and our self-constructed DocElemHard benchmark\footnote{\url{https://github.com/mohhao/DocPO}}, comprising 9,578 images. We employ three metrics to assess specific parsing modalities: Normalized Edit Distance (NED) for text, Character Detection Matching (CDM) for formulas, and Tree Edit Distance-based Similarity (TEDS) for tables. NED is reported on $[0,1]$ (lower is better), whereas CDM and TEDS are reported as percentages (higher is better). To provide a unified performance assessment, we calculate an overall metric defined as:

\begin{equation}
    \text{Overall} = \frac{(1 - \text{NED}) \times 100 + \text{TEDS} + \text{CDM}}{3}
    \label{eq:overall_score}
\end{equation}

More benchmark results appear in the supplementary material.

\begin{table*}[t]
\centering
\caption{Element evaluation on OmniDocBench and DocElemHard.}
\label{tab:comparison}
\small
\setlength{\tabcolsep}{3.6pt}
\resizebox{\textwidth}{!}{%
\begin{tabular}{lllllccccc|cccc}
\toprule
\multirow{2}{*}{\textbf{Type}} & \multirow{2}{*}{\textbf{Model}} & \multirow{2}{*}{\textbf{Post-proc.}} & \multirow{2}{*}{\textbf{RL}} & \multirow{2}{*}{\textbf{ViT Mod.}} & \multirow{2}{*}{\textbf{Size}} & \multicolumn{4}{c}{\textbf{OmniDocBench}} & \multicolumn{4}{c}{\textbf{DocElemHard}} \\
\cmidrule(lr){7-10} \cmidrule(lr){11-14}
 & & & & & & Overall$\uparrow$ & Text$\downarrow$ & Formula$\uparrow$ & Table$\uparrow$ & Overall$\uparrow$ & Text$\downarrow$ & Formula$\uparrow$ & Table$\uparrow$ \\
\midrule
\multirow{2}{*}{\shortstack[c]{General\\VLMs}}
 & Qwen2.5-VL-72B & No & No & -- & 72B & 89.95 & 0.0424 & 87.47 & 86.64 & 80.25 & 0.034 & 63.89 & 80.27 \\
 & Qwen2.5-VL-3B & No & No & -- & 3B & 88.05 & 0.0792 & 87.27 & 84.85 & 81.16 & 0.096 & 75.19 & 77.90 \\
\midrule
\multirow{6}{*}{\shortstack[c]{Specialized\\VLMs}}
 & dots.ocr & No & No & Arch+Pretrain & 3B & 89.60 & 0.034 & 90.40 & 81.90 & 87.54 & 0.037 & 85.52 & 80.80 \\
 & DeepSeek-OCR & No & No & Arch+Pretrain & 3B & 90.63 & 0.0350 & 93.43 & 81.95 & 86.30 & 0.0491 & 88.23 & 75.58 \\
 & PaddleOCR-VL & Yes & No & Arch+Pretrain & 0.9B & \underline{94.87} & \underline{0.0142} & \underline{94.10} & \underline{91.95} & \underline{91.47} & \underline{0.033} & \underline{90.82} & \underline{86.90} \\
\cmidrule(lr){2-14}
 & INFINITY Parser & No & Yes & \ding{55} & 7B & 89.13 & 0.025 & 81.20 & 88.70 & 87.50 & 0.043 & 86.70 & 80.10 \\
 & olmOCR 2 & No & Yes & \ding{55} & 7B & 93.05 & 0.0233 & 93.73 & 87.76 & 89.12 & \underline{0.033} & 89.38 & 81.29 \\
\cmidrule(lr){2-14}
 & \textbf{DocPO (Ours)} & No & Yes & \ding{55} & 3B & \textbf{95.49} & \textbf{0.0125} & \textbf{94.70} & \textbf{93.01} & \textbf{93.76} & \textbf{0.022} & \textbf{92.88} & \textbf{90.60} \\
\bottomrule
\multicolumn{14}{l}{\footnotesize \textbf{ViT Mod.}: Arch=architecture modification, Pretrain=additional ViT pre-training, \ding{55}=none.}
\end{tabular}%
}%
\end{table*}

\begin{table}[t]
\centering
\caption{Static exponent baselines vs. dynamic SAA on table recognition (OmniDocBench TEDS $\uparrow$).}
\label{tab:table_static_gamma}
\small
\setlength{\tabcolsep}{6pt}
\begin{tabular}{lc}
\toprule
\textbf{Method} & \textbf{OmniDocBench} \\
\midrule
APTED + fixed $\gamma=2$ & 92.1 \\
APTED + fixed $\gamma=4$ & 91.9 \\
APTED + fixed $\gamma=8$ & 92.3 \\
APTED + SAA & \textbf{93.01} \\
\bottomrule
\end{tabular}
\end{table}

\begin{table}[t]
\centering
\caption{Effect of the Dynamic Dispersion Controller (DDC) in SAA on table recognition (OmniDocBench TEDS $\uparrow$).}
\label{tab:table_ddc_ablation}
\small
\setlength{\tabcolsep}{6pt}
\begin{tabular}{lc}
\toprule
\textbf{Method} & \textbf{OmniDocBench} \\
\midrule
APTED + SAA (w/o DDC) & 92.73 \\
APTED + SAA & \textbf{93.01} \\
\bottomrule
\end{tabular}
\end{table}

\begin{figure}[t]
    \centering
    \includegraphics[width=1\columnwidth]{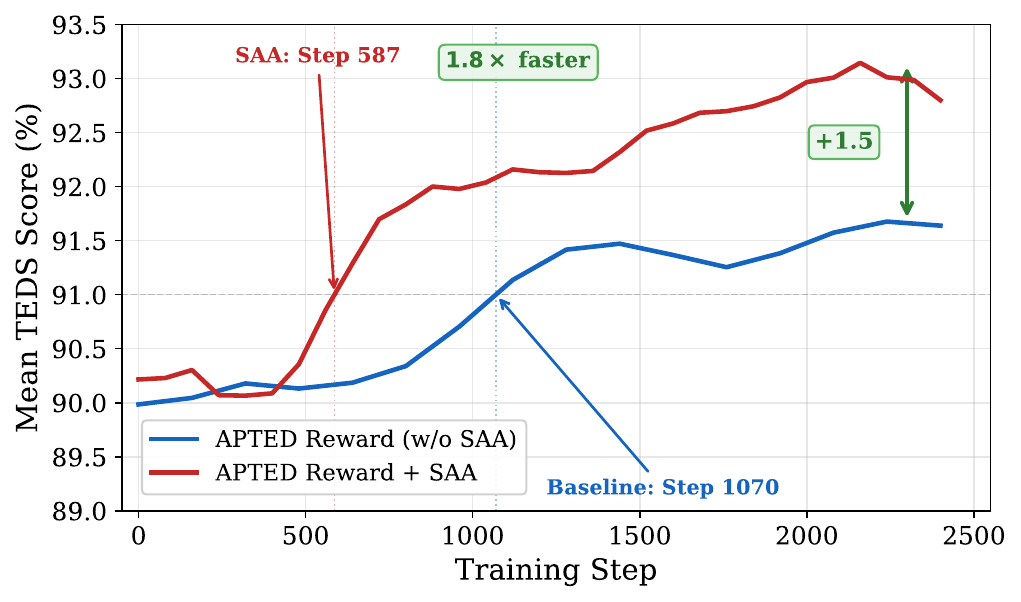}
    \caption{Evaluation TEDS score during training (table recognition). SAA reaches the 91.0 threshold $\sim$1.8$\times$ faster than the baseline, and maintains a +1.5 advantage at convergence.}
    \Description{Training curves comparing APTED reward with and without SAA on a held-out evaluation set. SAA reaches 91.0 TEDS at step 587 versus step 1070 for the baseline, and converges to approximately 93.0 versus 91.5.}
    \label{fig:training_curve}
\end{figure}

\begin{table}[t]
\centering
\caption{DocPO at different model scales. The same reward recipe transfers to 7B with modest, mixed changes relative to 3B.}
\label{tab:scaling_7b}
\small
\setlength{\tabcolsep}{3pt}
\begin{tabular}{lccc|ccc}
\toprule
 & \multicolumn{3}{c|}{\textbf{OmniDocBench}} & \multicolumn{3}{c}{\textbf{DocElemHard}} \\
\textbf{Scale} & Text$\downarrow$ & Formula$\uparrow$ & Table$\uparrow$ & Text$\downarrow$ & Formula$\uparrow$ & Table$\uparrow$ \\
\midrule
\shortstack[l]{Qwen2.5-VL-3B\\+ DocPO} & 0.0125 & 94.70 & \textbf{93.01} & 0.022 & \textbf{92.88} & 90.6 \\
\shortstack[l]{Qwen2.5-VL-7B\\+ DocPO} & \textbf{0.0120} & \textbf{94.94} & 92.83 & \textbf{0.019} & 92.00 & \textbf{90.8} \\
\bottomrule
\end{tabular}
\end{table}

\subsection{Ablation Study}

We conduct an ablation study with the primary goal of isolating the contribution of SAA. Specifically, we examine: (1) RL vs. SFT to quantify exploration benefits; (2) element-appropriate reference-based rewards and their key components; and (3) \textbf{Step-Aware Annealing} (SAA) vs. simpler non-annealed or fixed-curvature alternatives to measure the gain from progressive reward sharpening.

\textbf{Impact of RL Training (RL vs. SFT)}: 
Table~\ref{tab:performance} shows that our task-tailored RL approach improves over the SFT baseline across all three tasks. On OmniDocBench, text NED decreases from 0.0358 to 0.0238 with a plain edit-distance reward, and formula CDM improves from 92.61 to 92.86. The table task exhibits the same trend: moving from the SFT baseline (89.30) to RL with edit distance (90.05) already yields a clear gain, and replacing string-level matching with a structure-aware reward improves it further. This confirms that policy optimization can improve beyond static supervised training.

\textbf{Base Reward Instantiation and Component Ablations}: 
We compare simple reference-based rewards that match the native scoring criteria of each element: NED for text, APTED/TEDS for tables, and \textit{Rubric+edit} for formulas. The goal is not to claim a new reward family, but to establish strong and interpretable bases on top of which SAA can operate uniformly. The additional formula ablation shows that the syntax gate is necessary inside \textit{Rubric+edit}: removing it drops performance from 93.93 to 92.93 on OmniDocBench and from 87.69 to 86.61 on DocElemHard. This indicates that semantic consistency alone is insufficient when syntactically invalid LaTeX outputs are not explicitly suppressed.

\textbf{Effect of Step-Aware Annealing (SAA vs. Simpler Sharpening)}: 
Holding the underlying reward family fixed, we compare linear rewards, static power transforms, and the full Step-Aware Annealing schedule. For tables, Table~\ref{tab:table_static_gamma} shows that fixed exponents already improve over plain APTED, but none matches dynamic SAA: fixed $\gamma=2,4,8$ reach 92.1, 91.9, and 92.3 on OmniDocBench, all below the 93.01 achieved by SAA. Table~\ref{tab:table_ddc_ablation} further shows that removing the Dynamic Dispersion Controller (DDC) drops SAA from 93.01 to 92.73, indicating that the task-wise dispersion-aware schedule contributes additional gains beyond the step-aware schedule alone. Taken together, these results show that the improvement is not merely due to applying a larger constant nonlinearity; both the dynamic schedule and its DDC-based adaptation matter. The same pattern holds across other elements: for text, SAA reduces OmniDocBench NED from 0.0238 to \textbf{0.0125}, and for formulas, it improves \textit{Rubric+edit} from 93.93 to \textbf{94.70} on OmniDocBench and from 87.69 to \textbf{92.88} on DocElemHard. Overall, the benefit comes from sharper optimization under the same notion of correctness, rather than from changing the base reward itself.

\textbf{Training Efficiency}:
To further illustrate the effect of SAA on optimization dynamics, Figure~\ref{fig:training_curve} plots the mean TEDS score on a held-out evaluation set over training for both the APTED reward baseline and its SAA-enhanced variant. Both runs share the same base reward, training data, and hyperparameters; the only difference is the presence of SAA. As shown, the SAA variant crosses the 91.0 TEDS threshold at approximately step 587, whereas the baseline requires roughly 1,070 steps to reach the same level---a $\sim$1.8$\times$ speedup in convergence. Moreover, SAA maintains a consistent advantage throughout training, ultimately reaching $\sim$93.0\% while the baseline plateaus around $\sim$91.5\% (+1.5 points). This confirms that SAA not only accelerates learning by providing a stronger optimization signal, but also yields a higher final evaluation score.

\subsection{Evaluation Results on OmniDocBench and DocElemHard}

\textbf{Superior Performance without Post-Processing}:
Table~\ref{tab:comparison} compares our final unified model against leading general-purpose and specialized document VLMs on fine-grained element benchmarks. \textit{Crucially, our results are obtained directly from model outputs without model-specific post-processing.} Moreover, we emphasize a key architectural distinction: among the specialized baselines, dots.ocr scales the visual front-end with a 1.2B high-resolution encoder trained from scratch, DeepSeek-OCR introduces a custom DeepEncoder with window attention, convolutional compression, and global attention, and PaddleOCR-VL adopts a NaViT-style dynamic-resolution encoder initialized from Keye-VL~\citep{yang2025keyevl} with additional large-scale multimodal pretraining~\citep{cui2025paddleocrvl}. While these modifications are often well motivated, they may require extra adaptation and may reduce some deployment simplicity. In contrast, \textbf{our model uses the standard Qwen2.5-VL-3B backbone with no vision-encoder modification, no additional vision-encoder pretraining, and no post-processing}---the only change relative to the SFT baseline is the reward-level optimization introduced by DocPO. Under this simpler setup, our 3B model still surpasses all listed baselines. On the element-level OmniDocBench benchmark, our overall score reaches \textbf{95.49}. These results suggest that strong document parsing performance can also be obtained from a general-purpose backbone when the optimization signal is designed appropriately, rather than only through architectural customization. 

\textbf{State-of-the-Art on Complex Layouts}: The advantages of our approach are most pronounced in fine-grained element evaluation on challenging datasets. As shown in Table~\ref{tab:comparison}, our model achieves the lowest text edit distance on OmniDocBench (\textbf{0.0125}), surpassing all baselines including post-processed ones. Furthermore, on the challenging \textit{DocElemHard} benchmark, we secure the best \textbf{Overall} performance (\textbf{93.76}), outperforming the strongest listed specialized baseline, PaddleOCR-VL (91.47), with the best formula score (\textbf{92.88}). This superiority is evident in the demanding \textbf{Table} recognition task, where we score \textbf{90.6}, surpassing the nearest competitor (PaddleOCR-VL) by a substantial margin (+3.7). This demonstrates that our end-to-end method generalizes effectively to complex layouts without relying on model-specific post-processing.

\begin{table}[!t]
\centering
\caption{Subset-level table TEDS (\%) on OmniDocBench v1.5.}
\label{tab:table_drilldown}
\small
\setlength{\tabcolsep}{3pt}
\begin{tabular*}{\columnwidth}{@{\extracolsep{\fill}}llccc}
\toprule
\textbf{Property} & \textbf{Subset} & \textbf{DeepSeek-OCR} & \textbf{APTED} & \textbf{DocPO} \\
\midrule
Background & w/o bg & 83.4 & 93.2 & \textbf{93.9} \\
 & w/ bg & 78.7 & 88.9 & \textbf{90.2} \\
Equation & w/o eq. & 83.8 & 92.4 & \textbf{93.4} \\
 & w/ eq. & 72.8 & 89.4 & \textbf{89.7} \\
Language & English & 79.1 & 90.3 & \textbf{91.3} \\
 & En-Ch mixed & 93.0 & 94.1 & \textbf{94.6} \\
 & Chinese & 83.0 & 92.8 & \textbf{93.6} \\
Line style & fewer & 78.2 & 92.0 & \textbf{93.3} \\
 & full & 85.8 & 91.8 & \textbf{92.8} \\
 & less & 83.2 & \textbf{92.1} & 91.3 \\
 & no-line & 74.1 & 91.6 & \textbf{93.3} \\
Layout & horiz. & 82.8 & 92.1 & \textbf{93.0} \\
 & vert. & 7.3 & 75.0 & \textbf{78.6} \\
Cell span & no & 85.8 & 93.0 & \textbf{93.9} \\
 & yes & 73.3 & 89.3 & \textbf{90.3} \\
Struct. text & no & 82.9 & 92.7 & \textbf{93.01} \\
 & yes & 59.3 & 82.5 & \textbf{89.5} \\
\bottomrule
\end{tabular*}
\end{table}

\textbf{Fine-Grained Table Drill-Down}: Table~\ref{tab:table_drilldown} shows that DocPO beats DeepSeek-OCR on every subgroup and the non-annealed APTED baseline on 16 of 17. Its largest gains are on vertical layouts (+3.6 over APTED and +71.3 over DeepSeek-OCR), span-containing tables (+1.0 and +17.0), and structured-text tables (+7.0 and +30.2). Across line styles, DocPO is strongest on no-line and fewer-line tables; APTED is slightly better only on the less-line subset. The largest margins should be interpreted cautiously because the vertical and structured-text subsets are small ($n=6$ and $n=15$). Together with supplementary formula and text analyses, these results suggest that SAA helps most when difficult layouts require coordinated structure and content.

\subsection{Conclusion}

In this work, we introduced \textbf{Step-Aware Annealing} (SAA), a reward-sharpening mechanism instantiated within \textbf{DocPO}. By progressively sharpening fixed reference-based rewards, SAA improves text, table, and formula optimization without additional human supervision for reward construction. Using a general-purpose VLM without vision-encoder modifications, additional vision-encoder pretraining, or post-processing, DocPO achieves state-of-the-art results against systems with specialized architectures. Results on OmniDocBench and DocElemHard show that reward-level optimization complements model scaling and redesign, making SAA a practical mechanism for high-precision document understanding.

\section*{Limitations}

DocPO has three limitations. First, RL training is more expensive than SFT because of multiple rollouts, structural rewards, and formula rubric scoring; this offline cost does not affect inference. Second, proxy rewards can be gamed or miss subtle semantic nuances. Finally, we study only text, formulas, and tables; extending SAA to charts and geometric diagrams remains future work.

\clearpage
\balance

\clearpage
\appendix
\section*{Supplementary Material}
\balance

\section{Prompts for Document Element Parsing}
\label{sec:appendix_parsing_prompts}

The exact prompts used for text, formula, and table parsing during inference are as follows:

\begin{tcolorbox}[
    colback=gray!5, colframe=gray!50, arc=1mm, boxrule=0.5pt,
    left=2mm, right=2mm, top=2mm, bottom=2mm,
    fontupper=\small\ttfamily, halign=left, 
    before upper={\setlength{\parskip}{0.5em}} 
]
\textbf{\sffamily [Text Parsing]} \par
Parse the text block without using any \texttt{\$\$...\$\$}.\par

\textbf{\sffamily [Formula Parsing]} \par
Parse the formula with latex format.\par

\textbf{\sffamily [Table Parsing]} \par
Please convert this cropped image directly into html format of table.\par
\end{tcolorbox}

\section{Additional Analyses}
\label{sec:appendix_additional_analyses}

\subsection{Formula Rubric Prompt}
\label{sec:appendix_verifier_prompt}

Evaluating generated formulas requires distinguishing stylistic variations from semantic errors. We use the following prompt to instantiate the formula rubric:

\begin{tcolorbox}[
    colback=gray!5, colframe=gray!50, arc=1mm, boxrule=0.5pt,
    left=2mm, right=2mm, top=2mm, bottom=2mm,
    fontupper=\small\ttfamily, halign=left,
    before upper={\setlength{\parskip}{0.5em}}
]
Please determine whether \texttt{[Formula 1]} and \texttt{[Formula 2]} are semantically consistent. Ignore variations in representation (e.g., spacing or LaTeX command synonyms) and focus solely on semantic equivalence. \par

The evaluation must be strict, including identical variable names. If \texttt{[Formula~2]} contains abnormal repetitions, it should be deemed \textbf{Inconsistent}. \par

\textbf{\sffamily [Input Format]} \par
\texttt{[Formula 1]} = """\{gt\}""" \par
\texttt{[Formula 2]} = """\{pred\}""" \par
\par
\textbf{\sffamily [Output]} \par
Respond only with \textbf{Consistent} or \textbf{Inconsistent}.
\end{tcolorbox}

\subsection{Rubric Reliability}
We instantiate a lightweight \textit{text-only} rubric for formulas with \textbf{Qwen2.5-7B-Instruct} in a \textit{zero-shot} setting, using the prompt above. To reduce training overhead, we cache the rubric output for each $(y^*, y)$ pair and reuse it across repeated rollouts.

We evaluate rubric reliability on 200 formula pairs randomly sampled from the test set. We manually annotate their semantic equivalence and compare the annotations with rubric outputs under the same prompt. The rubric achieves high accuracy, supporting its use as a lightweight semantic signal. This post-hoc study is used only for analysis, not for reward tuning, checkpoint selection, or model selection.

\begin{table}[htbp]
\centering
\caption{Accuracy of the formula rubric on 200 randomly sampled formula pairs.}
\label{tab:llm_judge_acc}
\small
\begin{tabular}{lc}
\toprule
\textbf{Rubric} & \textbf{Accuracy} \\
\midrule
Qwen2.5-7B-Instruct (zero-shot) & 97.5\% \\
\bottomrule
\end{tabular}
\end{table}

\subsection{Comparison with Peer RL Parsers}
\label{sec:supp_peer_rl}
To isolate the optimization method from full-page pipeline effects, we compare DocPO with peer RL parsers under the same element-level protocol. All models receive the same cropped patches and their official element-specific prompts, with no additional post-processing. Table~\ref{tab:supp_peer_rl} reports the three element metrics and the aggregate score defined in the main paper.

\begin{table*}[!t]
\centering
\caption{Comparison with peer RL parsers under the same element-level protocol. Text is measured by NED ($\downarrow$); formula, table, and overall scores are percentages ($\uparrow$).}
\label{tab:supp_peer_rl}
\small
\setlength{\tabcolsep}{3.6pt}
\resizebox{0.98\textwidth}{!}{%
\begin{tabular}{lllccccc||cccc}
\toprule
\multirow{2}{*}{\textbf{Model}} & \multirow{2}{*}{\textbf{RL method}} & \multirow{2}{*}{\textbf{Post-proc.}} & \multirow{2}{*}{\textbf{Size}} & \multicolumn{4}{c}{\textbf{OmniDocBench}} & \multicolumn{4}{c}{\textbf{DocElemHard}} \\
\cmidrule(lr){5-8} \cmidrule(lr){9-12}
 & & & & Overall$\uparrow$ & Text$\downarrow$ & Formula$\uparrow$ & Table$\uparrow$ & Overall$\uparrow$ & Text$\downarrow$ & Formula$\uparrow$ & Table$\uparrow$ \\
\midrule
INFINITY Parser & Edit-distance RL & \ding{55} & 7B & 89.13 & 0.0250 & 81.20 & 88.70 & 87.50 & 0.043 & 86.70 & 80.10 \\
olmOCR 2 & Binary unit-test reward & \ding{55} & 7B & 93.05 & 0.0233 & 93.73 & 87.76 & 89.12 & 0.033 & 89.38 & 81.29 \\
\midrule
\textbf{DocPO} & Element-specific + SAA & \ding{55} & 3B & \textbf{95.49} & \textbf{0.0125} & \textbf{94.70} & \textbf{93.01} & \textbf{93.76} & \textbf{0.022} & \textbf{92.88} & \textbf{90.60} \\
\bottomrule
\end{tabular}%
}
\end{table*}

Despite using a 3B rather than a 7B backbone, DocPO improves the OmniDocBench overall score by 2.44 points over olmOCR 2 and by 6.36 points over INFINITY Parser. The advantage is especially clear for tables: DocPO reaches 93.01 on OmniDocBench and 90.60 on DocElemHard. The corresponding scores are 87.76 and 81.29 for olmOCR 2, and 88.70 and 80.10 for INFINITY Parser. This matched comparison indicates that the gains are associated with element-specific continuous rewards and SAA rather than differences in input granularity or post-processing.

\subsection{Alternative Reward Shaping and Optimization Stability}
\label{sec:supp_alternative_shaping}
We compare SAA with three generic shaping alternatives using the same 80k-sample OmniDocBench setup. Let $M\in[0,1]$ be the base reward and $G$ the rollout-group size. Rank shaping maps the within-group rank to $(G-\operatorname{rank})/(G-1)$; margin shaping linearly rescales rewards above 0.8; and sigmoid shaping introduces a center and a slope $k$.

\begin{table}[H]
\centering
\caption{Reward-shaping alternatives on the 80k-sample OmniDocBench setup (aggregate score $\uparrow$).}
\label{tab:supp_reward_shaping}
\small
\setlength{\tabcolsep}{5pt}
\begin{tabular}{llc}
\toprule
\textbf{Shaping} & \textbf{Function $g(M)$} & \textbf{Score} \\
\midrule
Rank & $(G-\operatorname{rank})/(G-1)$ & 88.32 \\
Margin & $\max(0,M-0.8)/0.2$ & 90.53 \\
Sigmoid & $\sigma(k(M-0.8))$ & 91.05 \\
\textbf{SAA (DocPO)} & $M^{\gamma(s)}$ & \textbf{92.18} \\
\bottomrule
\end{tabular}
\end{table}

SAA exceeds the strongest generic alternative, sigmoid shaping, by 1.13 points, and improves over margin and rank shaping by 1.65 and 3.86 points. Thus, generic nonlinear shaping is useful, but SAA performs best among the tested choices without requiring a manually selected margin threshold or sigmoid center and slope. Unlike temperature schedules that rescale action logits, SAA directly reshapes the normalized task reward while preserving its ordering.

\begin{figure}[htbp]
\centering
\includegraphics[width=\columnwidth]{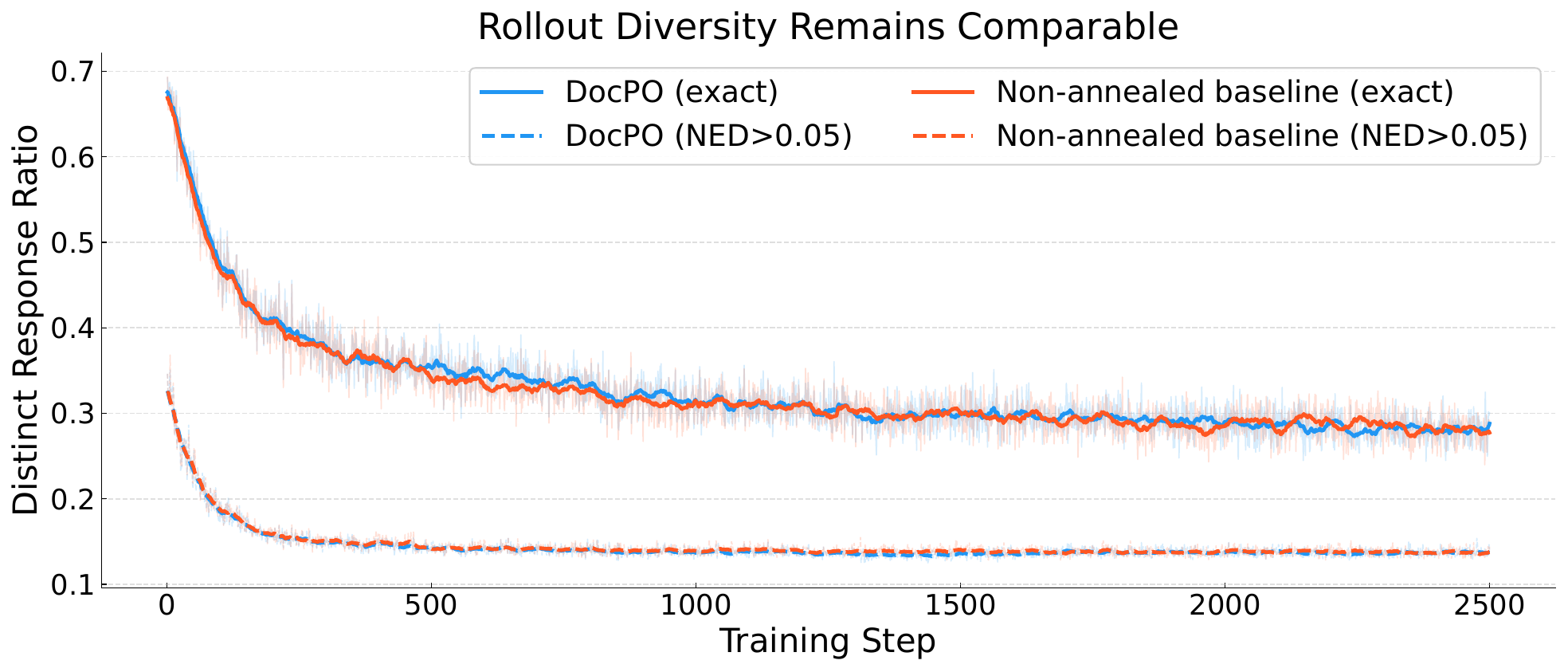}
\vspace{-4pt}
\includegraphics[width=\columnwidth]{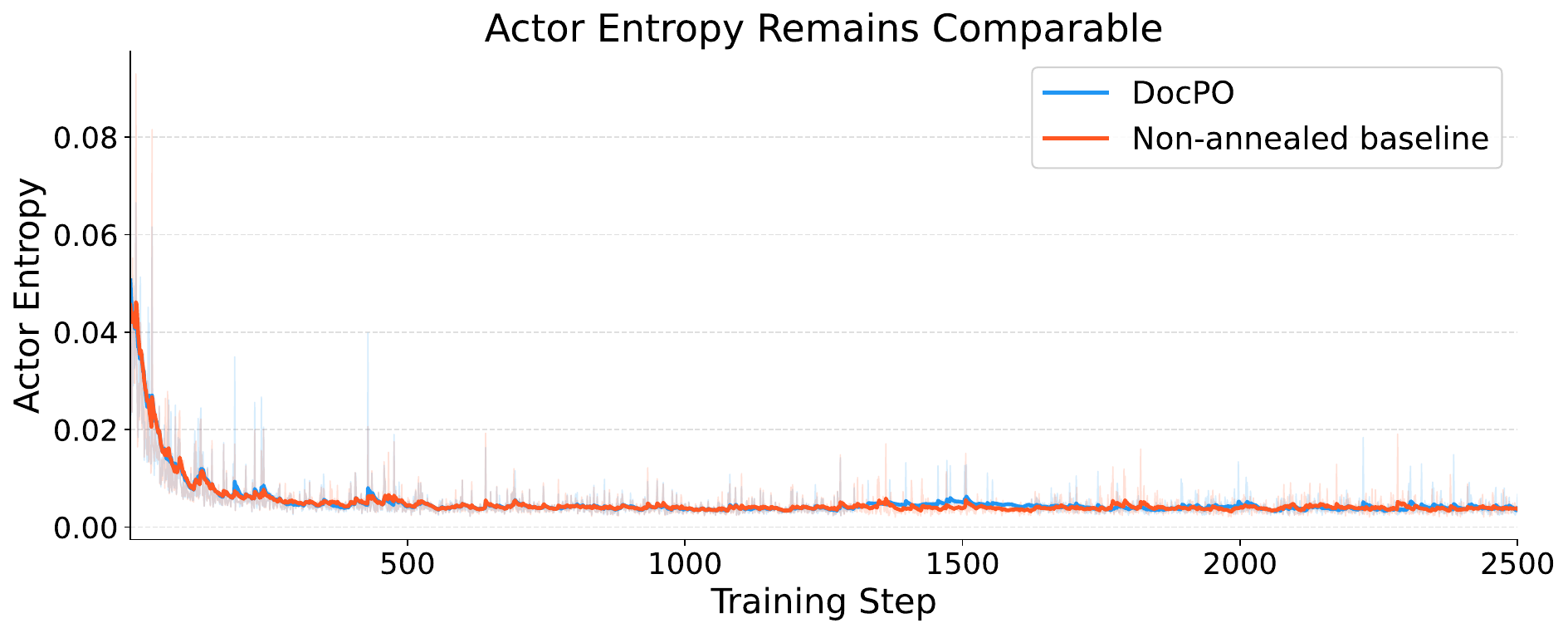}
\caption{Optimization-stability diagnostics for DocPO and the non-annealed baseline. Top: rollout diversity measured by distinct-response ratios under exact matching (solid) and NED $>0.05$ (dashed). Bottom: actor entropy. The curves remain closely aligned across training.}
\Description{Two training plots compare DocPO with a non-annealed baseline. In the top plot, exact-match and NED-threshold rollout-diversity curves nearly overlap for the two methods. In the bottom plot, actor-entropy curves also nearly overlap throughout training.}
\label{fig:supp_stability}
\end{figure}

Figure~\ref{fig:supp_stability} checks whether progressively sharper rewards cause premature policy collapse. DocPO and the non-annealed baseline follow comparable rollout-diversity and actor-entropy trajectories, with no additional collapse relative to the baseline evident in either diagnostic. Thus, the performance gain is not accompanied by a visible loss of exploration under these measurements.

\textbf{Training cost.} Using 32 H20 GPUs, a global batch size of 128, a 12k maximum sequence length, and 8 rollouts, standard GRPO requires 60.3 seconds per step versus 60.8 for DocPO (0.8\% overhead). SFT requires 10.8 seconds per step, confirming that rollout generation dominates the offline RL cost. SAA does not affect inference latency; structural reward computations are parallelized across CPUs to avoid a GPU-side bottleneck.

\subsection{Additional Benchmark Results}
\label{sec:supp_more_benchmarks}
We report additional results on two public benchmarks for table and formula recognition, comparing DeepSeek-OCR, the non-annealed DocPO variant (w/o SAA), and the full DocPO model.

\begin{table}[htbp]
\centering
\caption{Additional benchmark results on WikiTableSet~\citep{icpram23} and UniMERNet~\citep{wang2024unimernetuniversalnetworkrealworld}. Higher is better for all metrics.}
\label{tab:supp_additional_benchmarks}
\small
\setlength{\tabcolsep}{4pt}
\begin{tabular}{@{}lll ccc@{}}
\toprule
\textbf{Benchmark} & \textbf{Subset} & \textbf{Metric} & \textbf{DeepSeek-OCR} & \textbf{w/o SAA} & \textbf{DocPO} \\
\midrule
\multirow{2}{*}{WikiTableSet}
  & --  & TEDS           & 83.8 & 92.1 & \textbf{95.8} \\
  & --  & TEDS-S         & 88.6 & 97.2 & \textbf{98.7} \\
\addlinespace[3pt]
\multirow{2}{*}{UniMERNet}
  & CPE & CDM            & 78.2 & 88.7 & \textbf{95.3} \\
  & HWE & CDM            & 84.7 & 92.4 & \textbf{94.7} \\
\bottomrule
\end{tabular}
\end{table}

These results are consistent with the main-paper conclusions. On WikiTableSet, DocPO improves over DeepSeek-OCR by 12.0 TEDS and 10.1 TEDS-S, and exceeds the non-annealed variant by 3.7 and 1.5 points, respectively. On UniMERNet, DocPO improves formula recognition on both challenging subsets, outperforming DeepSeek-OCR by 17.1 CDM on CPE and 10.0 on HWE. It also retains an advantage over the version without SAA: 6.6 points on CPE and 2.3 on HWE. Overall, the pattern observed on OmniDocBench and DocElemHard carries over to these external benchmarks: the reward design and step-aware annealing are most beneficial on structurally difficult or handwriting-heavy cases.

\subsection{OmniDocBench v1.5 Drill-Down for Formulas and Text Blocks}
\label{sec:supp_omnidocbench_drilldown}
To complement the table drill-down in the main paper, we further analyze page-level subsets for formulas and text blocks on OmniDocBench v1.5. We report CDM ($\uparrow$) for formulas and normalized edit distance ($\downarrow$) for text blocks.

\begin{table}[htbp]
\centering
\caption{Page-level formula drill-down on OmniDocBench v1.5 (CDM $\uparrow$).}
\label{tab:supp_formula_drilldown}
\small
\setlength{\tabcolsep}{4pt}
\resizebox{\columnwidth}{!}{%
\begin{tabular}{llccc}
\toprule
\textbf{Property} & \textbf{Subset} & \textbf{DeepSeek-OCR} & \textbf{w/o SAA} & \textbf{DocPO} \\
\midrule
Overall & ALL & 91.4 & 91.9 & \textbf{93.5} \\
Visual noise & None & 92.7 & 93.7 & \textbf{94.4} \\
 & Colorful bg. & 87.8 & 87.4 & \textbf{92.9} \\
 & Fuzzy scan & 72.9 & 82.1 & \textbf{95.2} \\
 & Watermark & 73.7 & 79.1 & \textbf{88.0} \\
Language & English & 94.1 & 94.2 & \textbf{95.5} \\
 & Chinese & 83.6 & 85.3 & \textbf{87.8} \\
Layout & 1+ columns & 96.2 & 93.5 & \textbf{97.6} \\
 & Double column & 93.0 & 93.5 & \textbf{95.2} \\
 & Other layout & 79.7 & 86.1 & \textbf{88.1} \\
 & Single column & 90.3 & 91.1 & \textbf{91.7} \\
 & Three column & 93.8 & 99.2 & \textbf{99.3} \\
Source & Academic literature & 94.3 & 85.8 & \textbf{96.2} \\
 & Book & 88.6 & 89.0 & \textbf{90.0} \\
 & PPT2PDF & 89.3 & 91.3 & \textbf{91.7} \\
 & Colorful textbook & 93.4 & 96.8 & \textbf{96.9} \\
 & Exam paper & 93.8 & 94.3 & \textbf{96.0} \\
 & Note & 90.1 & \textbf{100.0} & \textbf{100.0} \\
\bottomrule
\end{tabular}%
}
\end{table}

DocPO improves over the non-annealed formula variant on nearly every reported subset, with the largest gains on visually difficult pages: 13.1 CDM on fuzzy scans, 8.9 on watermark pages, 10.4 on academic literature pages, and 5.5 on colorful backgrounds. Its margins over DeepSeek-OCR are larger on these hard slices, reaching 22.3 CDM on fuzzy scans and 14.3 on watermark pages. The only subset where DocPO does not exceed the non-annealed baseline is the small \textit{note} subset, where both methods are already saturated at 100.0.

\begin{table}[htbp]
\centering
\caption{Page-level text-block drill-down on OmniDocBench v1.5 (EditDist $\downarrow$).}
\label{tab:supp_text_drilldown}
\scriptsize
\setlength{\tabcolsep}{3pt}
\resizebox{\columnwidth}{!}{%
\begin{tabular}{llccc}
\toprule
\textbf{Property} & \textbf{Subset} & \textbf{DeepSeek-OCR} & \textbf{w/o SAA} & \textbf{DocPO} \\
\midrule
Overall & ALL & 0.0351 & 0.0238 & \textbf{0.0125} \\
Visual noise & None & 0.0180 & 0.0160 & \textbf{0.0078} \\
 & Colorful bg. & 0.0430 & 0.0280 & \textbf{0.0153} \\
 & Fuzzy scan & 0.1005 & 0.0549 & \textbf{0.0200} \\
 & Watermark & 0.0940 & 0.0775 & \textbf{0.0328} \\
Language & En-Ch mixed & 0.0559 & 0.0311 & \textbf{0.0171} \\
 & English & 0.0170 & 0.0130 & \textbf{0.0053} \\
 & Chinese & 0.0498 & 0.0333 & \textbf{0.0189} \\
Layout & 1+ columns & 0.0192 & 0.0165 & \textbf{0.0046} \\
 & Double column & 0.0362 & 0.0246 & \textbf{0.0128} \\
 & Other layout & 0.0414 & 0.0264 & \textbf{0.0144} \\
 & Single column & 0.0357 & 0.0243 & \textbf{0.0137} \\
 & Three column & 0.0325 & 0.0205 & \textbf{0.0087} \\
Source & Academic literature & 0.0230 & 0.0082 & \textbf{0.0051} \\
 & Book & 0.0185 & 0.0180 & \textbf{0.0124} \\
 & PPT2PDF & 0.0293 & 0.0144 & \textbf{0.0113} \\
 & Colorful textbook & 0.0515 & 0.0515 & \textbf{0.0196} \\
 & Exam paper & 0.0687 & 0.0525 & \textbf{0.0184} \\
 & Magazine & 0.0092 & 0.0057 & \textbf{0.0042} \\
 & Newspaper & 0.0204 & 0.0123 & \textbf{0.0091} \\
 & Note & 0.0810 & 0.0429 & \textbf{0.0274} \\
 & Research report & 0.0094 & 0.0026 & \textbf{0.0024} \\
\bottomrule
\end{tabular}%
}
\end{table}

The text-block drill-down shows the same pattern: DocPO outperforms both DeepSeek-OCR and the non-annealed baseline on every reported subset. The gains are particularly large on noisy or structurally difficult pages. Compared with the non-annealed variant, DocPO reduces edit distance by 0.0349 on fuzzy scans, 0.0447 on watermark pages, 0.0319 on colorful textbooks, and 0.0341 on exam papers. The corresponding reductions relative to DeepSeek-OCR are 0.0805, 0.0612, 0.0319, and 0.0503. These results match the table drill-down in the main paper and support a unified conclusion across all three element types: DocPO is most advantageous when layout complexity, visual corruption, or language mixing makes fine-grained discrimination especially important.

\subsection{Hyperparameter Sensitivity}
We assess the sensitivity of Step-Aware Annealing to the maximum curvature adjustment range $\Delta_{\gamma}$ using the overall score. We set $\gamma_{\text{init}}=1$ and vary $\Delta_{\gamma}$ while keeping other settings fixed.

\begin{table}[htbp]
\centering
\caption{Sensitivity to $\Delta_{\gamma}$ on element-level parsing (Overall $\uparrow$).}
\label{tab:delta_gamma_sensitivity}
\small
\begin{tabular}{lcc}
\toprule
\textbf{$\Delta_{\gamma}$} & \textbf{OmniDocBench} & \textbf{DocElemHard} \\
\midrule
0 (w/o SAA) & 94.42 & 90.53 \\
4  & 94.87 & 91.65 \\
8 (DocPO)   & \textbf{95.49} & \textbf{93.76} \\
\bottomrule
\end{tabular}
\end{table}

We next study the reward mixing weights $\alpha$ and $\beta$ in the formula reward $R_{\text{formula}} = v_{\text{syn}} \cdot (\alpha \cdot r_{\text{sem}} + \beta \cdot r_{\text{struct}})$. Table~\ref{tab:alpha_beta_sensitivity} reports formula CDM on OmniDocBench under different $(\alpha, \beta)$ combinations without SAA. The two boundary cases ($\alpha{=}0$: NED only; $\alpha{=}1$: rubric only) both underperform mixed settings, confirming that the semantic and structural signals are complementary. The chosen setting, $\alpha{=}0.8$ and $\beta{=}0.2$, achieves the highest score.

\begin{table}[htbp]
\centering
\caption{Sensitivity to $(\alpha, \beta)$ in the formula reward on OmniDocBench (CDM $\uparrow$, without SAA).}
\label{tab:alpha_beta_sensitivity}
\small
\begin{tabular}{ccc}
\toprule
\textbf{$\alpha$} & \textbf{$\beta$} & \textbf{OmniDocBench} \\
\midrule
0   & 1.0 & 92.86 \\
0.5 & 0.5 & 92.76 \\
0.8 & 0.2 & \textbf{93.93} \\
0.9 & 0.1 & 93.58 \\
1.0 & 0   & 93.01 \\
\bottomrule
\end{tabular}
\end{table}

The base time constant $\tau_{\text{base}}$ governs the characteristic time scale of the exponential schedule. We default to roughly half of the total training steps, so that $\gamma$ completes most of its growth in the first half of training and gradually saturates thereafter. DDC then modulates this pace via $\tau_{\text{adaptive}} = \tau_{\text{base}} / (1 + d_{\text{DDC}})$.

\subsection{Inline Formulas in Text Blocks}
In end-to-end document parsing, text blocks may include inline mathematical expressions. The OmniDocBench protocol evaluates text blocks with string-level normalized edit distance (NED); we therefore use NED as the corresponding reward.

\subsection{Data Filtering Details}
\label{sec:appendix_filtering_details}
Before RL training, we perform a pre-filtering pass with 8 rollouts per sample to identify uninformative samples. Specifically, we remove samples that achieve perfect scores in all 8 rollouts because they offer no room for improvement. We also remove samples that receive zero reward in all 8 rollouts; their zero variance indicates consistent failure and little useful learning signal. This lightweight screening step reduces unnecessary computation and focuses RL optimization on samples where the model can meaningfully improve.

\textbf{Evaluation protocol and checkpoint selection.} No final benchmark test set is used for early stopping or model selection. The stabilization criterion in the main paper refers to a held-out validation split, composed of hard cases and separated from the training pool before RL. It covers mixed-script or noisy text, span-heavy tables, and long or multi-line formulas. OmniDocBench and DocElemHard are used only for final reporting.

\textbf{Leakage prevention.} We remove overlaps and near-duplicates between the training and validation pools and the final benchmarks. Image patches are screened with perceptual hashing; candidate textual matches are retrieved with SimHash over normalized parsed text and then verified with NED. This two-stage procedure filters both visually duplicated patches and textually near-identical content.

\subsection{Analysis of Step-Aware Annealing}
\label{sec:appendix_theory_saa}
We provide a formal analysis of how increasing $\gamma$ over training sharpens the reward signal among near-correct samples. We first establish two propositions about the power-law transform $f_{\gamma}(M)=M^{\gamma}$, then discuss how they connect to the GRPO advantage and the annealing schedule.

\textbf{Setup.} Let $M\in(0,1]$ denote a base reward (normalized to $[0,1]$, following the main paper). The shaped reward is $f_{\gamma}(M) = M^{\gamma}$ with $\gamma \ge 1$. In GRPO, the advantage for sample $i$ within a rollout group $\mathcal{G}$ is computed as $A_i = (r_i - \mu_{\mathcal{G}}) / \sigma_{\mathcal{G}}$, where $r_i = M_i^{\gamma}$ is the shaped reward.

\begin{proposition}[Relative Margin Amplification]
\label{prop:relative_margin}
For $m\in(0,1)$ and small $\delta>0$ with $m+\delta\le 1$, define the relative gap $\rho_{\gamma}(m,\delta) \triangleq \frac{f_{\gamma}(m+\delta) - f_{\gamma}(m)}{f_{\gamma}(m)}$. Then $\rho_{\gamma}(m,\delta)$ is monotonically increasing in $\gamma$ for all $m\in(0,1)$.
\end{proposition}

\begin{proof}
By direct computation,
\begin{equation}
\rho_{\gamma}(m,\delta) = \frac{(m+\delta)^{\gamma} - m^{\gamma}}{m^{\gamma}} = \left(1+\frac{\delta}{m}\right)^{\gamma} - 1.
\end{equation}
Let $\alpha = 1+\delta/m > 1$. Then $\rho_{\gamma} = \alpha^{\gamma}-1$ and $\frac{\partial \rho_{\gamma}}{\partial \gamma} = \alpha^{\gamma}\log\alpha > 0$ since $\alpha>1$. \qedhere
\end{proof}

This shows that the \textit{relative} reward difference between two nearby candidates always grows with $\gamma$. For small $\delta$, the \textit{absolute} gap is approximately $\gamma m^{\gamma-1}\delta$ and can shrink for moderate $m$ because $m^{\gamma-1}\to 0$ when $m<1$. The relative gap, however, quantifies the increasing pre-normalization separation, while the monotonic transform preserves the original reward ordering. The effect is strongest when $\delta/m$ is large, that is, when the gap is already meaningful relative to the baseline score.

\begin{proposition}[Exponential Concentration of Reward Mass]
\label{prop:concentration}
For two candidates with rewards $0<M_2<M_1\le 1$ in the same rollout group, the reward ratio satisfies
\begin{equation}
\frac{f_{\gamma}(M_1)}{f_{\gamma}(M_2)} = \left(\frac{M_1}{M_2}\right)^{\gamma},
\end{equation}
which grows exponentially in $\gamma$. Consequently, the normalized reward mass $q_i = f_{\gamma}(M_i)/\sum_j f_{\gamma}(M_j)$ over a rollout group converges to a point mass on $\arg\max_i M_i$ as $\gamma\to\infty$.
\end{proposition}

\begin{proof}
The ratio identity follows directly because $M_1/M_2>1$. For the concentration result, note that for any $j\neq i^{*}$ where $i^{*}=\arg\max_i M_i$, we have $q_j / q_{i^{*}} = (M_j/M_{i^{*}})^{\gamma} \to 0$ as $\gamma\to\infty$, so $q_{i^{*}} \to 1$. \qedhere
\end{proof}

\textbf{Connection to GRPO Advantage.} The propositions above describe the pre-normalization reward profile. GRPO subsequently centers and scales rewards within each group through $A_i = (r_i - \mu_{\mathcal{G}})/\sigma_{\mathcal{G}}$, so they do not by themselves guarantee monotonic growth in the normalized advantage. They show instead that nonlinear shaping changes the relative spacing of raw rewards and concentrates their mass on better candidates. Empirically, this sharper profile can produce more decisive within-group weighting during policy updates.

\textbf{Why Anneal (Coarse-to-Fine Curriculum).} When $\gamma=1$, $f_{\gamma}(M)=M$ is the identity and all reward differences contribute proportionally, yielding stable early training with broad gradient signal. As $\gamma$ increases, Propositions~\ref{prop:relative_margin}--\ref{prop:concentration} imply that the optimization signal progressively concentrates on distinguishing near-correct candidates. This implements a coarse-to-fine curriculum: early steps move rewards upward broadly; later steps refine among high-quality outputs.

\textbf{Role of the Dynamic Dispersion Controller (DDC).} DDC defines a task-wise dispersion score $d_{\text{DDC}} = \frac{\sigma}{\mu + \epsilon}$ over a recent reward window, instantiated with the rolling coefficient of variation. Concretely, the adaptive time constant $\tau_{\text{adaptive}}=\tau/(1+d_{\text{DDC}})$ modulates the annealing pace based on this internal control signal. When $d_{\text{DDC}}$ is large (dispersed rewards, indicating the model still produces highly variable outputs), $\tau_{\text{adaptive}}$ decreases and $\gamma$ grows faster, applying stronger sharpening earlier. When $d_{\text{DDC}}$ is small (rewards already concentrated), the schedule stays closer to the default pace, avoiding premature over-sharpening that could reduce gradient diversity.

\section{Dataset Statistics and Comparisons}
\label{sec:dataset_details}

We provide a statistical breakdown of the \textit{DocElemHard} dataset. DocElemHard is constructed from diverse document sources---academic papers, books, and web-captured pages---to stress-test element-level parsing on genuinely difficult instances. We collect raw element patches across categories, perform $k$-means clustering on visual features to ensure diversity, and then apply category-specific difficulty filters described below. As shown in Table~\ref{tab:dataset_stats}, the final dataset covers three core element categories totaling 9{,}578 instances.

To prevent benchmark leakage, we remove overlaps and near-duplicates between the training/validation data and the test benchmarks using perceptual hashing for images and SimHash retrieval followed by normalized edit-distance verification for text. We will release DocElemHard and the evaluation protocol at \url{https://github.com/mohhao/DocPO} upon publication to support reproducibility.

\begin{table}[ht]
\centering
\caption{Overall distribution of document elements in the proposed \textit{DocElemHard} dataset ($N=9,578$).}
\label{tab:dataset_stats}
\small
\begin{tabular}{lr}
\toprule
\textbf{Element Category} & \textbf{Instance Count} \\
\midrule
Text Block & 8,100 \\
Formula    & 480 \\
Table      & 998 \\
\bottomrule
\end{tabular}
\end{table}

\begin{table}[htbp]
\centering
\caption{Structural and linguistic comparison between \textit{DocElemHard}$_{\text{table}}$ and \textit{OmniDocBench}$_{\text{table}}$.}
\label{tab:dataset_comparison}
\small
\setlength{\tabcolsep}{5pt}
\begin{tabular}{@{}lcc@{}}
\toprule
 & \textbf{DocElemHard} & \textbf{OmniDocBench} \\
\midrule
Total Samples            & 998           & 512           \\
\addlinespace[3pt]
\multicolumn{3}{@{}l}{\textit{Language}} \\
\quad English            & 998 (100\%)   & 196 (38.3\%)  \\
\quad Chinese (Simp.)    & --            & 295 (57.6\%)  \\
\quad Mixed              & --            & 21\;\:(4.1\%) \\
\addlinespace[3pt]
\multicolumn{3}{@{}l}{\textit{Equation}} \\
\quad w/ Embedded Eq.    & \textbf{609 (61.0\%)} & 88\;\:(17.2\%) \\
\quad Text-only          & 389 (39.0\%)  & 424 (82.8\%)  \\
\addlinespace[3pt]
\multicolumn{3}{@{}l}{\textit{Complexity}} \\
\quad w/ Spanning Cells  & \textbf{714 (71.5\%)} & 159 (31.1\%) \\
\quad Regular Layout     & 284 (28.5\%)  & 353 (68.9\%)  \\
\bottomrule
\end{tabular}
\end{table}

\textbf{Table subset} (998 samples): We retain only tables that are structurally challenging: each table must contain more than 80 cells, include at least one spanning cell (\texttt{rowspan} or \texttt{colspan}), or exhibit degraded visual quality (scan artifacts, partial occlusion). Table~\ref{tab:dataset_comparison} highlights the resulting structural complexity compared to OmniDocBench---notably, 61.0\% of our tables contain embedded equations and 71.5\% have spanning cells, far exceeding the corresponding OmniDocBench proportions.

\textbf{Formula subset} (480 samples): We accumulate failure cases from iterative model evaluation, retaining formulas that consistently cause errors across multiple model versions. The resulting subset is dominated by structurally complex expressions: 20.2\% contain multi-line environments (\texttt{aligned}, \texttt{cases}, \texttt{array}), 14.6\% feature nested structures (e.g., fractions within summations), and 24.4\% exceed 200 characters in LaTeX source length. Many originate from scanned documents with low contrast or print artifacts.

\textbf{Text block subset} (8,100 samples): Similarly sourced from accumulated bad cases across diverse documents (academic papers, books, and web pages). Content-wise, 31.4\% contain bold or italic markers requiring style-aware recognition, 14.9\% include inline math or LaTeX fragments, 11.7\% are multi-line passages, and 68.5\% contain digits or numerical expressions. These factors combine to create a benchmark where models must handle mixed formatting and noisy visual conditions simultaneously.

\begingroup
\renewcommand{\refname}{Supplementary References}

\endgroup

\end{document}